%% file: non_overlap_arxiv.tex
\documentclass[10pt]{article}
\usepackage{verbatim}
\usepackage{graphicx}
\usepackage[margin=1.2in]{geometry}
\usepackage{amsmath}
\usepackage{amssymb}
\usepackage{amsthm}
\usepackage{algorithm2e}
\usepackage{color}
\usepackage{xr}
\usepackage{url}
\usepackage{wrapfig}

\usepackage{natbib}

\newcommand{\bs}{\mathbf}
\newtheorem{thm}{Theorem}[section]
\newtheorem{cor}[thm]{Corollary}
\newtheorem{prop}[thm]{Proposition}
\newtheorem{lem}[thm]{Lemma}
\input{macros}

\title{Globally Optimal Gradient Descent for a ConvNet with Gaussian Inputs}
\author{
        Alon Brutzkus \\ \texttt{alonbrutzkus@mail.tau.ac.il}
            \and
       Amir Globerson \\
       \texttt{gamir@cs.tau.ac.il}
}
	\date{\vspace{-5ex}}

\begin{document}
\maketitle

\begin{abstract} 
	Deep learning models are often successfully trained using gradient descent, despite the worst case hardness of the underlying non-convex optimization problem. The key question is then under what conditions can one prove that optimization will succeed. Here we provide a strong result of this kind. We consider a neural net with one hidden layer and a convolutional structure with no overlap and a ReLU activation function. For this architecture we show that learning is NP-complete in the general case, but that when the input distribution is Gaussian, gradient descent converges to the global optimum in polynomial time. To the best of our knowledge, this is the first global optimality guarantee of gradient descent on a convolutional neural network with ReLU activations. 
\end{abstract}
%\tableofcontents
%\pagebreak

\section{Introduction}
\label{Introduction}
Deep neural networks have achieved state-of-the-art performance on many  machine learning tasks in areas such as natural language processing \citep{WuSCLNMKCGMKSJL16}, computer vision \citep{krizhevsky2012imagenet} and speech recognition \citep{hinton2012deep}. Training of such networks is often successfully performed by minimizing a high-dimensional non-convex objective function, using simple first-order methods such as stochastic gradient descent. 

Nonetheless, the success of deep learning from an optimization perspective is poorly understood theoretically. Current results are mostly pessimistic, suggesting that even training a 3-node neural network is NP-hard \citep{blum1993training}, and that the objective function of a single neuron can admit exponentially many local minima \citep{auer1996exponentially,safran2015quality}. There have been recent attempts to bridge this gap between theory and practice.  Several works focus on the geometric properties of loss functions that neural networks attempt to minimize. For some simplified architectures, such as linear activations, it can be shown that there are no bad local minima \citep{kawaguchi2016deep}. Extension of these results to the non-linear case currently requires very strong independence assumptions \citep{kawaguchi2016deep}.

Since gradient descent is the main ``work-horse'' of deep learning it
is of key interest to understand its convergence properties. However,
there are no results showing that gradient descent is globally optimal
for non-linear models, except for the case of many hidden neurons
\citep{valiant2014learning} and non-linear activation functions that
are not widely used in practice \citep{zhang2017electron}.\footnote{See
	more related work in \secref{Related_Work}.}  Here we provide the
first such result for a neural architecture that has two very common
components: namely a ReLU activation function and a convolution layer.
%It still remains unknown under which conditions it is possible to learn neural networks that are common in practice, such as convolutional neural networks with ReLU activations.

The architecture considered in the current paper is shown in \figref{non_overlap_fig}. We refer to these models as \textit{no-overlap} networks. A no-overlap network can be viewed as a simple convolution layer with non overlapping filters, followed by a ReLU activation function, and then average pooling.  Formally, let $\ww\in\reals^m$ denote the filter coefficient, and assume the input $\xx$ is in $\reals^d$. Define $k=m/d$ and assume for simplicity that $k$ is integral. Partition $\xx$ into $k$ non-overlapping parts and denote $\xpart{i}$ the $i^{th}$ part. Finally, define $\sigma$ to be the ReLU activation function, namely $\relu{z} = \max\{0,z\}$.
Then the output of the network in \figref{non_overlap_fig} is given by:
\be
f(\xx;\ww) = {1\over k} \sum_{i} \relu{\ww\cdot\xpart{i} }
\label{eq:nonoverlap_function}
\ee
We note that such architectures have been used in several works \citep{lin2013network,milletari2016v}, but we view them as important firstly because they capture key properties of general convolutional networks.

We address the realizable case, where training data is generated from a function as in \eqref{eq:nonoverlap_function} with weight vector $\ww^*$.  Training data is then generated by sampling $n$ training points $\xx_1,\ldots,\xx_n$ from a distribution $\cD$, and assigning them labels using $y=f(\xx;\ww^*)$.  The learning problem is then to find a $\ww$ that minimizes the squared loss. In other words, solve the optimization problem:
\be
\min_{\ww} {1\over n}  \sum_i \left( f(\xx_i;\ww) - y_i \right) ^2
\label{eq:empirical_risk}
\ee
In the limit $n\to\infty$, this is equivalent to minimizing the population risk:
\be
\ell(\ww) = \expect{\xx\sim \cD}{\left( f(\xx;\ww) - f(\xx;\ww^*) \right)^2}
\label{population_risk}
\ee
Like several recent works \citep{hardt2016gradient,hardt2016identity} we focus on minimizing the population risk, leaving the finite sample case to future work. We believe the population risk captures the key characteristics of the problem, since the large data regime is the one of interest. 
\ignore{
	In this paper, we study the tractability of learning a two layer non-linear convolutional neural networks with gradient descent, in the realizable case. In our setting, we receive observations from a non-linear convolutional neural network $f(\mathbf{x})=\mathbf{b}^T\sigma(A^*\mathbf{x})$ where $A^* \in \mathbb{R}^{d \times n}$ corresponds to a single-channel convolutional layer with shared weight values across neurons, $\sigma$ is the element-wise ReLU activation function ($\sigma(\mathbf{x})_i = \max\{0,\mathbf{x}_i\}$) and $\mathbf{b} \in \mathbb{R}^n$ corresponds to a mean-pooling layer, i.e., $\mathbf{b}_i=\frac{1}{n}$ for all $1 \leq i \leq n$. Our goal is to fit a two-layer convolutional neural network $h(\mathbf{x})=\mathbf{b}^T\sigma(A\mathbf{x})$, with the same architecture as $f$, to the observations. Towards that end, we consider minimizing the population risk:
	
	\begin{equation}
	\label{population_risk}
	\mathbb{E}_{\mathbf{x} \sim \mathcal{D}}[(\mathbf{b}^T\sigma(A\mathbf{x}) - \mathbf{b}^T\sigma(A^*\mathbf{x}))^2]
	\end{equation} 
}
%where $\mathcal{D}$ is the input distribution.
%	\resizebox{0.5\textwidth}{!}{\includegraphics{non_overlap2}}
%centerline{\includegraphics[width=\columnwidth]{non_overlap2}}
\begin{wrapfigure}{r}{6.5cm}
		\includegraphics[width=6.5cm]{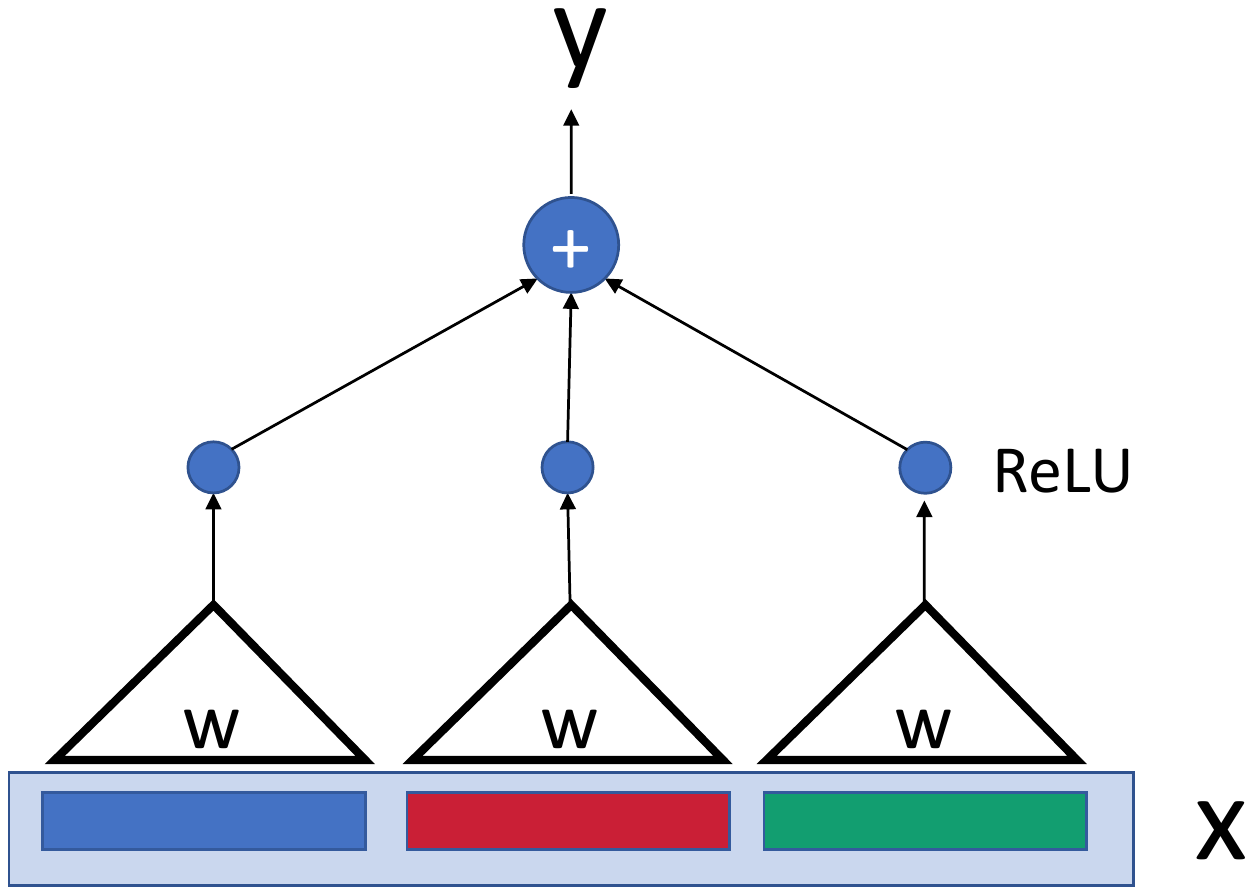}		
		\caption{Convolutional neural network with non-overlapping filters. In the first layer, a filter $\ww$ is applied to non-overlapping parts of the input vector $\xx$, and the output passes through a ReLU activation function. The outputs of the neurons are then averaged to give the output $y$.}
		\label{non_overlap_fig}
\end{wrapfigure} 

Our key results are as follows:
\begin{itemize}
	\item {\bf Worst Case Hardness: } Despite the simplicity of {\ourarch}, we show that learning them is in fact hard if $\cD$ is unconstrained. Specifically, in \secref{hardness_section}, we show
	that learning {\ourarch} is NP complete via a reduction from a variant of the set splitting problem.
	\item {\bf Distribution Dependent Tractability: } When $\cD$ corresponds to independent Gaussian variables with $\mu=0,\sigma^2=1$, we show in \secref{non_overlapping_gaussian} that {\ourarch} can be learned in polynomial time using gradient descent. 
\end{itemize}
%For \textit{non-overlap} networks we give both negative and positive results (Section~\ref{hardness_section} and Section~\ref{non_overlapping_gaussian}, respectively). On the negative side, we prove that approximately optimizing the objective \eqref{population_risk} is NP-Complete. In contrast, we show that under the assumption that each entry of the input data is i.i.d. gaussian with mean 0, a randomly initialized gradient descent almost surely converges to the unique global minimum of the objective \eqref{population_risk}. This is despite the presence of a saddle point in the objective. Moreover, we show that gradient descent gets arbitrarily close to the global minimum in polynomial time. 
The above two results nicely demonstrate the gap between worst-case intractability and tractability under assumptions on the data. We provide an empirical demonstration of this in \secref{nonoverlap_experiments} where gradient descent is shown to succeed on the Gaussian case and fail for a different distribution.

%In Section~\ref{nonoverlap_experiments} we provide a clear empirical demonstration of our results. 
\ignore{
	show that our theory predicts cases where the optimization of neural networks is hard and when it is tractable. Concretely, we optimize an objective of a network with non-overlapping filters with the the worst-case data set constructed in the hardness result in Section~\ref{hardness_section} and with a dataset with i.i.d. gaussian entries. Both datasets are of the same size and labeled by the same neural network that we attempt to learn. We show that given the worst-case dataset, the optimization fails and gets stuck in a sub-optimal point, whereas given the gaussian data, it converges to the global minimum and thus in this case we are able to learn the network parameters.  
}

To further understand the role of overlap in the network, we consider
networks that do have overlap between the filters. In \secref{overlapping_section} we show that in this case, even under Gaussian distributed inputs, there will be non-optimal local minima.  Thus, gradient descent will no longer be optimal in the overlap case. In \secref{general_conv_experiments} we show empirically that these local optima may be overcome in practice by using gradient descent with multiple restarts.

\ignore{
	In addition, in the case where the data follows a gaussian distribution, we consider the general case, with no assumption on the overlapping of the filters ( Section~\ref{overlapping_section}). We show that the previous result for the non-overlapping case does not hold in general, namely, gradient descent does \textit{not} converge almost surely to the global minimum of the objective \eqref{population_risk}. We prove this by constructing a network with overlapping filters such that a randomly initialized gradient descent will get stuck in a sub-optimal region with probability at least $\frac{1}{4}$.
}

\ignore{
	Giving convergence guarantees of gradient descent in the general setting under the Gaussian assumption, remains an intriguing open problem. We conduct comprehensive experiments in this setting which suggest that by restarting gradient descent a constant number of times, one can ensure that with high probability gradient descent converges to the global optimum.
}

Taken together, our results are the first to demonstrate distribution dependent optimality of gradient descent for learning a neural architecture with a convolutional like architecture and a ReLU activation function.  

\section{Related Work} 
\label{Related_Work}
Hardness of learning neural networks has been demonstrated for many different settings. For example, \citep{blum1993training} show that learning a neural network with one hidden layer with a sign activation function is NP-hard in the realizable case. \citep{livni2014computational} extend this to other activation functions and bounded norm optimization. Hardness can also be shown for improper learning under certain cryptographic assumptions \citep[e.g., see][]{daniely2014average,klivans2008cryptographic,livni2014computational}. Note that these hardness results do not hold for the regression and tied parameter setting that we consider.  

%\citet{shamir2016distribution} shows that one needs to employ assumptions on both the input distribution and the target function in order to obtain learning guarantees. This is in line with our work where we make such assumptions.
Due to the above hardness results, it is clear that the success of deep-learning can only be explained by making additional assumptions about the data generating distribution. The classic algorithm by \citep{baum1990polynomial} shows that intersection of halfspaces (i.e., a specific instance of a one hidden layer network) is PAC learnable under any symmetric distribution. This was later extended in \citep{klivans2009baum} to log-concave distributions. 

The above works do not consider gradient descent as the optimization method, leaving open the question of which assumptions can lead to global optimality of gradient descent. Such results have been hard to obtain, and we survey some recent ones below. One instance when gradient descent can succeed is when there are enough hidden units such that random initialization of the first layer can lead to zero error even if only the second layer is trained. Such over-specified networks have been considered in \citep{valiant2014learning,livni2014computational} and it was shown that gradient descent can globally learn them in some cases \citep{valiant2014learning}. However, the assumption of over-specification is very restrictive and limits generalization. In contrast, we show convergence of gradient descent to a global optimum for any network size and consider convolutional neural networks with shared parameters. Another interesting case is linear dynamical systems, where \citep{hardt2016gradient} show that under independence assumptions maximum likelihood is quasi-concave and hence solvable with gradient ascent. 

Recent work by \citep{mei2016landscape} shows that regression with a {\em single neuron} and certain non-linear activation functions, can be learned with gradient descent for sub-Gaussian inputs. We note that their architecture is significantly simpler than ours, in that it uses a single neuron. In fact, their regression problem can also be solved via methods for generalized linear models such as \citep{NIPS2011_4429}. 

%For networks with linear activations, several recent works have shown that 

%In \cite{valiant2014learning} they prove that gradient descent can learn polynomials with sufficiently large two-layer fully-connected networks. Their lower bounds on the network size depend inversely on the loss error, and can get extremely large as the error approaches 0.
%It was recently shown that gradient descent (or stochastic variants thereof) converge to a global optimum if the optimization landscape contains only local minma.

%\footnote{We note that even in this case, it wasn't shown that gradient descent can be used for learning.}   

\citep{shamir2016distribution} recently showed that there is a limit to what distribution dependent results can achieve. Namely, it was shown that for large enough one-hidden layer networks, no distributional assumptions can make gradient descent tractable. Importantly, the construction in \citep{shamir2016distribution} does not use parameter tying and thus is not applicable to the architecture we study here. 

\ignore{
	Recently, the theory of neural networks has drawn significant attention. Nonetheless, we know of only a few that study the dynamics of gradient descent in optimization of neural networks. In a work most similar to ours, \citep{zhang2017electron} prove that gradient descent converges to the global minimum on two-layer neural networks for different activation functions. However, they do not cover the case where the activation is the ReLU function. In
	\citep{valiant2014learning} they prove that gradient descent can learn polynomials with sufficiently large two-layer fully-connected networks. Their lower bounds on the network size depend inversely on the loss error, and can get extremely large as the error approaches 0. 
}

Several works have focused on understanding the loss surface of neural network objectives, but without direct algorithmic implications. \citep{kawaguchi2016deep} show that {\em linear} neural networks do not suffer from bad local minima.   \citep{hardt2016identity} consider objectives of {\em linear residual} networks and prove that there are no critical points other than the global optimum.   \citep{soudry2016no} show that in the objective of over-parameterized neural networks with dropout-like noise, all differentiable local minima are global. Other works \citep{safran2015quality,haeffele2015global} give similar results for over-specified networks. All of these results are purely geometric and do not have direct implications on convergence of optimization algorithms. In a different approach, \citep{janzamin2015beating}, suggest alternatives to gradient-based methods for learning neural networks. However, these algorithms are not widely used in practice. Finally, \citep{choromanska2015loss} use spin glass models to argue that, under certain generative modelling and architectural constraints, local minima are likely to have low loss values. 

The theory of non-convex optimization is closely related to the theory of neural networks. Recently, there has been substantial progress in proving convergence guarantees of simple first-order methods in various machine learning problems, that don't correspond to typical neural nets. These include for example matrix completion \citep{ge2016matrix} and tensor decompositions \citep{ge2015escaping}. 
%Nonetheless, such results have not been established for non-convex optimization problems that arise in training of non-linear neural networks.

Finally, recent work by \citep{ZhangBHRV16} shows that neural nets can perfectly fit random labelings of the data. Understanding this from an optimization perspective is largely an open problem.   

\section{Preliminaries} 
\label{preliminaries}
We use bold-faced letters for vectors and capital letters for matrices. The $i^{th}$ row of a matrix $A$ is denoted by $\mathbf{a}_i$.

In our analysis in Section \ref{non_overlapping_gaussian} and Section \ref{overlapping_section} we assume that the input feature $\xx \in \reals^d$ is a vector of IID Gaussian random variables with zero mean and variance one.\footnote{The variance per variable can be arbitrary. We choose one for simplicity.} Denote this distribution by $\cG$.  We consider networks with one hidden layer, and $k$ hidden units. Our main focus will be on {\ourarch}, but we begin with a more general one-hidden-layer neural network with a fully-connected layer parameterized by $W\in \reals^{k,d}$ followed by average pooling. The network output is then:
\be
f(\xx;W) = {1\over k} \sum_{i} \relu{\ww_i\cdot\xx }
\label{eq:nonoverlap_function_fully_connected}
\ee
where $\relu{}$ is the pointwise ReLU function.

We consider the realizable setting where there exists a {\em true} $W^*$ using which the training data is generated. The population risk (see \eqref{population_risk}) is then:
\be
\ell(W) = \expect{\cG}{(f(\xx;W) - f(\xx;W^*))^2} ~,
\ee
As we show next, $\ell(W)$ can be considerably simplified. First, define:
\be
g(\uu,\vv) = \expect{\cG}{\relu{\uu \cdot \xx}\relu{\vv \cdot \xx}}
\label{eq:def_g}
\ee
Simple algebra then shows that:
\be
\ell(W)  =\frac{1}{k^{2}} \sum_{i,j}\left[{g(\ww_i,\ww_j)} - 2{g(\ww_i,\ww^*_j)} + {g(\ww^*_i,\ww^*_j)}\right]
\ee

The next Lemma from \citep{cho2009kernel} shows that $g(\uu,\vv)$ has a simple form. 
\begin{lem}[\citep{cho2009kernel}, Section 2]
	Assume $\xx\in\reals^d$ is a vector where the entries are IID Gaussian random variables with mean 0 and variance 1. Given $\uu,\vv \in \reals^d$ denote by $\theta_{\uu,\vv}$ the
	angle between $\uu$ and $\vv$. Then:
	\be
	g(\uu,\vv)= \frac{1}{2\pi}\left \| \uu \right \| \left \| \vv \right \|\Bigg(\sin \theta_{\uu,\vv} + \Big(\pi - \theta_{\uu,\vv} \Big) \cos \theta_{\uu,\vv} \Bigg) \nonumber
	\ee
\end{lem}
%\begin{proof}
%	See Section 2 in \citet{cho2009kernel}.
%\end{proof}
\ignore{
	Therefore the loss function is given by 
	\begin{equation}
	\label{general_loss_function}
	l(A)  = \frac{1}{n^2}\Big[ \sum_{i,j}{g(A_i,A_j)} - 2\sum_{i,j}{g(A_i,A^*_j)} + \sum_{i,j}{g(A^*_i,A^*_j)} \Big]
	\end{equation}
	
	where 
	
	\begin{equation}
	\label{kernel_function}
	g(\ww,\vv) \triangleq \frac{1}{2\pi}\left \| \ww \right \| \left \| \vv \right \|\Bigg(\sin \theta + \Big(\pi - \theta \Big) \cos \theta \Bigg)
	\end{equation}
	
	and $\theta$ is the angle between $\ww$ and $\vv$. Note that $g(\ww,\ww) = \frac{1}{2}{\left \| \ww \right \|}^2$.
}
The gradient of $g$ with respect to $\uu$ also turns out to have a simple form, as stated in the lemma below. The proof is deferred to the Appendix \ref{prelim_supp}.
\begin{lem}
	\label{kernel_gradient}
	Let $g$ be as defined in \eqref{eq:def_g}. Then	$g$ is differentiable at all points $\uu \neq \mathbf{0}$ and $$\frac{\partial g(\uu,\vv)}{\partial \uu} = \frac{1}{2\pi}\left \| \vv \right \| \frac{\uu}{\left \| \uu \right \|} \sin \theta_{\uu,\vv} + \frac{1}{2\pi}\Big(\pi - \theta_{\uu,\vv} \Big) \vv$$
\end{lem}

We conclude by special-casing the results above to {\ourarch}. In this case, the entire model is specified by a single {\em filter} vector $\ww\in\reals^m$. The rows $\ww_i$ are mostly zeros, except for the indices $((i-1)m +1,\ldots,im)$ which take the values of $\ww$. Namely, $\ww_i = \left(\mathbf{0}_{(i-1)m},\ww,\mathbf{0}_{d-im}\right)$ where $\mathbf{0}_{l} \in \mathbb{R}^l$ is a zero vector. The same holds for the vectors $\ww^*_i$ with a weight vector $\ww^*$. This simplifies the loss considerably, since for all $i$:  $g(\ww_i,\ww_i)=\frac{1}{2}\left\|\ww\right\|^2$, and for all $i\neq j$: $g(\ww_i,\ww_j)=\frac{1}{2\pi}\left\|\ww\right\|^2$ and  $g(\ww_i,\ww^*_j)=\frac{1}{2\pi}\left\|\ww\right\|\left\|\ww^*\right\|$. Thus the loss $\ell(\ww)$ for {\ourarch}  yields (up to additive factors in $\ww^*$):
\begin{equation}
\label{non_overlap_loss_eq}
%\resizebox{.2 \textwidth}{!} 
%\begin{split} 
%l(\ww)&= \frac{1}{k^2} \Big[\big(\frac{k}{2} + \beta(k)\big){\left \|\ww \right \|}^2 - 2kg(\ww,\ww^*) \\ &-2\beta(k)\left \|\ww \right \|\left \|\ww^* \right \| \Big]
l(\ww)= \frac{1}{k^2} \Big[\gamma{\left \|\ww \right \|}^2 - 2kg(\ww,\ww^*) -2\beta\left \|\ww \right \|\left \|\ww^* \right \| \Big]
%\end{split}
\end{equation}
where $\beta = \frac{k^2-k}{2\pi}$ and $\gamma = \beta + \frac{k}{2}$.

\ignore{
	\begin{equation}
	\label{non_overlap_loss_eq}
	%\resizebox{.2 \textwidth}{!} 
	\begin{split} 
	l(\ww)&= \frac{1}{n^2} \Big[\big(\frac{n}{2} + \frac{n^2-n}{2\pi}\big){\left \|\ww \right \|}^2 - 2ng(\ww,\ww^*) \\ &-(n^2-n)\frac{\left \|\ww \right \|\left \|\ww^* \right \|}{\pi} \Big]
	\end{split}
	\end{equation}
}

\ignore{
	\be
	\ell(\ww) =
	\frac{2\pi - 1 + n}{4}\|\ww\|^2 - 2\pi g(\ww,\ww^*) 
	+ (1-n){\left \|\ww \right \|\left \|\ww^* \right \|} 
	\label{eq:gaussian_loss}
	\ee
}

\ignore{
	\bea
	\ell(\ww)& =&
	\frac{1}{n^2} \Big[\big(\frac{n}{2} + \frac{n^2-n}{2\pi}\big){\left \|\ww \right \|}^2 - 2ng(\ww,\ww^*) \\
	&&  -(n^2-n)\frac{\left \|\ww \right \|\left \|\ww^* \right \|}{\pi} +  c(\ww^*)\Big]
	\eea
}
%where $c(\ww^*)$ is a function of $\ww^*$ only, and thus irrelevant for optimizing over $\ww$.

\ignore{
	%\hspace*{-11cm}
	\bea
	&& \expect{\cG}{(\mathbf{b}^T\sigma(A\xx))^2} - 2\expect{\cD}{(\mathbf{b}^T\sigma(A\xx))(\mathbf{b}^T\sigma(A^*\xx))} \\
	&&  + \expect{\cG}{(\mathbf{b}^T\sigma(A^*\xx))^2}
	\eea
	and
	\begin{equation}
	\begin{split}
	\mathbb{E}[(\mathbf{b}^T\sigma(A\xx))^2] &=  \sum_{i,j}{\mathbf{b}_i\mathbf{b}_j\mathbb{E}[\sigma(A_i \cdot \xx)\sigma(A_j \cdot \xx)}]
	\end{split}
	\end{equation}
	Hence it suffices to calculate $\mathbb{E}_\xx[\sigma(\ww \cdot \xx)\sigma(\vv \cdot \xx)]$ for all vectors $\ww,\vv \in \mathbb{R}^d$. This integral has a closed from expression.
}

\section{Learning {\ourarch} is NP-Complete}
\label{hardness_section}
The {\ourarch} architecture is a simplified convolutional layer with average pooling. However, as we show here, learning it is still a hard problem. This will motivate our exploration of distribution dependent results in \secref{non_overlapping_gaussian}.

Recall that our focus is on minimizing the squared error in \eqref{population_risk}. For this section, we do not make any assumptions on $\cD$. Thus $\cD$ can be a distribution with uniform mass on training points $\xx_1,\ldots,\xx_n$, recovering the empirical risk in \eqref{eq:empirical_risk}. We know that $\ell(\ww)$ in \eqref{population_risk} can be minimized by setting $\ww=\ww^*$ and the corresponding squared loss $\ell(\ww)$ will be zero. However, we of course do not know $\ww^*$, and the question is how difficult is it to minimize $\ell(\ww)$. In what follows we show that this is hard. Namely, it is an NP-complete problem to find a $\ww$ that comes $\epsilon_0$ close to the minimum of $\ell(\ww)$, for some constant $\epsilon_0$. %, via reduction from a variant of the Set-Splitting problem \cite{Garey:1990}.  

We begin by defining the \textit{Set-Splitting-by-k-Sets} problem, which is a variant of the classic Set-Splitting problem \citep{Garey:1990}. After establishing the hardness of \textit{Set-Splitting-by-k-Sets}, we will provide a reduction from it to learning {\ourarch}.
\begin{definition}
	The \textit{Set-Splitting-by-k-Sets} decision problem is defined as follows:
	Given a finite set $S$ of $d$ elements and a collection $\cal C$ of at most $(k-1)d$ subsets $C_j$ of $S$, do there exist disjoint sets $S_1,S_2,...,S_k$ such that $\bigcup_i{S_i}=S$ and for all $j$ and $i$, $C_j \not\subseteq S_i$?
\end{definition}
For $k=2$ and without the upper bound on $|\cal C|$ this is known as the \textit{Set-Splitting} decision problem which is NP-complete \citep{Garey:1990}. Next, we show that \textit{Set-Splitting-by-k-Sets} is NP-complete. The proof is via a reduction from 3SAT and induction, and is provided in Appendix \ref{nonoverlap_hardness_supp}.
%The next two propositions show that \textit{Set-Splitting-by-k-Sets} is NP-complete. We first show $k=2$, and then use it to show for all $k$.

\ignore{
	\begin{prop}
		\textit{Set-Splitting-by-2-Sets} is NP-complete.
	\end{prop}
	\begin{proof}
		First, define an \textit{Equal-3SAT} problem to be a 3SAT problem with the additional restriction that the number of clauses and variables is the same. Then  \textit{Equal-3SAT} can be shown to be NP-complete via a reduction from 3SAT.\footnote{By adding clauses of the form ($w$) and ($x\vee y\vee z$) for new variables $w,x,y,z$ one can construct from a formula $\phi$ a formula $\psi$ with equal number of variables and clauses such that $\phi$ is satisfiable if and only if $\psi$ is.}
		Furthermore, the reduction from 3SAT to \textit{Set-Splitting} creates a set $S$ with $2n+1$ elements and $\cal C$ with $m+n$ subsets where $n$ and $m$ are the number of variables and clauses in a given CNF formula, respectively (\textcolor{red}{show reference}). Hence, this translates to a reduction from \textit{Equal-3SAT} to \textit{Set-Splitting-by-2-Sets} (in \textit{Set-Splitting-by-2-Sets} we require that $|\cal C| \leq |S|$).
	\end{proof}
}
\begin{prop}
	\label{prop:set-splitting}
	\textit{Set-Splitting-by-k-Sets} is NP-complete for all $k\geq 2$.
\end{prop}
\ignore{
	\begin{proof}
		We will show a reduction from \textit{Set-Splitting-by-k-Sets} to \textit{Set-Splitting-by-(k+1)-Sets}. Given $S=\{1,2,...,n\}$ and $\cal C=\{C_j\}_j$ define $S'=\{1,2,...,n+1\}$ and $\cal C' = \cal C \cup \{d_j\}_j$ where $d_j = \{j,n+1\}$ for all $1\leq j \leq n$. Note that $|\cal C'| \leq kn < k(n+1)$.
		Assume that there are $S_1,...,S_k$ that split the sets in $\cal C$. Then if we define $S_{k+1} = \{n+1\}$, it follows that $\bigcup_{i=1}^{k+1}{S_i}=S$ and $S_1,...,S_k,S_{k+1}$ are disjoint and split the sets in $\cal C'$.
		
		Conversely, assume that $S_1,...,S_k,S_{k+1}$ split the sets in $\cal C'$. Let w.l.o.g. $S_{k+1}$ be the set that contains $n+1$. Then for all $1 \leq j \leq n$ we have $d_j \not\subseteq S_{k+1}$. It follows that for all $1 \leq j \leq n$, $j \notin S_{k+1}$, or equivalently, $S_{k+1} = \{n+1\}$. Hence, $\bigcup_{i=1}^k{S_i}=S$ and $S_1,...,S_k$ are disjoint and split the sets in $\cal C$, as desired.
	\end{proof}
}

\newcommand{\kopt}{k-Non-Overlap-Opt}
We next formulate the {\ourarch} optimization problem.
\begin{definition}
	The {\kopt} problem is defined as follows. The input is a distribution $\cD_{X,Y}$ over input-output pairs $\xx,y$ where $\xx \in \mathbb{R}^d$.
	%training data $(\xx_i,y_i)$ generated via $y_i=f(\xx_i;\ww^*)$ where $f(\xx_i;\ww^*)$ is the output of a {\ourarch} and $\ww^*$ is a weight vector. 
	If the input is realizable by a \textit{no-overlap} network with $k$ hidden neurons, then the output is a vector $\ww$ such that:
	\be
	\expect{\cD_{X,Y}}{\left(f(\xx;\ww)-y)\right)^2} < \frac{1}{4k^5d}
	\label{eq:kopt}
	%{1\over n} \sum_i \left(f(\xx_i;\ww)-y_i)\right)^2 < \frac{1}{4k^5n}
	\ee
	Otherwise an arbitrary weight vector is returned.
\end{definition}
The above problem returns a $\ww$ that minimizes the population-risk up to $\frac{1}{4k^5d}$ accuracy. It is thus easier than minimizing the risk to an arbitrary precision $\epsilon$ (see \secref{non_overlapping_gaussian}, Theorem \ref{non_overlap_poly_conv}).

\ignore{
	Denote by \textit{$k$-non-overlap} network, a \textit{non-overlap} network with $k$ hidden neurons. Define the following problem \textit{k-nonoverlap-training-approximation}:
	
	"Let $\{\mathbf{x}_i\}$ be a set of training examples, where $\mathbf{x}_i \in \mathbb{R}^n$ are distributed according to a distribution $\mathcal{D}$ and labels $f(\mathbf{x}_i) \in \mathbb{R}$. If $f$ is a \textit{$k$-nonoverlap} network, output a \textit{$k$-nonoverlap} network $h$ such that $$\mathbb{E}_{\mathbf{x}\sim  \mathcal{D}}[(h(\mathbf{x})-f(\mathbf{x})]^2 < \frac{1}{4k^5n}$$ 
	Otherwise, output an arbitrary \textit{$k$-nonoverlap} network."
}

We prove the following theorem, which uses some ideas from \citep{blum1993training}, but introduces additional constructions needed for the no overlap case.
\begin{thm}
	For all the $k \geq 2$, the {\kopt} problem is NP-complete.
\end{thm}
\begin{proof}
	We will show a reduction from \textit{Set-Splitting-by-k-sets} to {\em{\kopt}}. Assume a given instance of the \textit{Set-Splitting-by-k-sets} problem with a set $S$ and collection of subsets $\cal C$. Denote $S=\{1,2,...,d\}$ and $|\cal C|$ $\leq (k-1)d$.  Let $\zero_d \in \reals^d$ be the all zeros vector. For a vector $\vv \in \reals^d$, define the vector $\dd_i(\vv) \in \reals^{kd}$ to be the concatenation of $i-1$ vectors $\zero_d$, followed by $\vv$ and $k-i$ vectors $\zero_d$, and let $\dd(\vv) = (\dd_1(\vv),\dd_2(\vv),...,\dd_k(\vv)) \in \reals^{k^2d}$. 
	
	We next define a training set for {\em{\kopt}}. For each element $i\in S$ define an input vector $\xx_i = \dd(\eee_i)$, where $\eee_i$ is the standard basis of $\reals^d$. Assign the label $y_i={1\over k}$ to this input.  In addition, for each subset $C_j\in \cal C$ define the vector $\xx_{d+j} = \dd(\sum_{i\in C_j}{\eee_i})$ and label $y_{d+j} = 0$. Thus we have $|S|+|\cal C|$ inputs in $\reals^{k^2d}$. Let $\cD_{X,Y}$ be a uniform distribution over the training set points (i.e., each point with probability at least $\frac{1}{kd}$ since $|\cal C|$ $\leq (k-1)d$).
	
	\ignore{
		training set on the $k^2d$-dimensional hypercube $\{0,1\}^{k^2d}$. For each $i \in  S$ define the vector $\mathbf{d}(\mathbf{e}_i) \in \mathbb{R}^{k^2d}$ where $\mathbf{e}_i \in \mathbb{R}^d$ is $i$th one-hot vector, with label $\frac{1}{k}$. In addition, for each $C_j$ define the vector $\mathbf{d}(\sum_{i\in C_j}{\mathbf{e}_i}) \in \mathbb{R}^{k^2d}$ with label $0$. Denote by $f$ the label function, i.e., $f(\mathbf{d}(\mathbf{e}_i)) = \frac{1}{k}$ and  $f(\mathbf{d}(\sum_{i\in C_j}{\mathbf{e}_i})) = 0$ for all $i$ and $j$. We let $\mathcal{D}$ be a uniform distribution over the training set points (i.e., each point with probability at least $\frac{1}{kd}$ since $|\cal C| \leq (k-1)d$).
	}	
	We will now show that the given instance of \textit{Set-Splitting-by-k-sets} has a solution (i.e., there exist splitting sets) if and only if {\em{\kopt}} returns a weight vector with low risk. First, assume there exist splitting sets $S_1,...,S_k$. 
	For each $1 \leq l \leq k$ define the vector $\aa^{S_l} \in \reals^d$ such that for all $i \in S_l$, $a^{S_l}_i = 1$ and $a^{S_l}_i = -d$ otherwise. Define a {\ourarchsngl} with $k^2d$ inputs and weight vector $\ww = (\aa^{S_1}, \aa^{S_2},...,\aa^{S_k}) \in \reals^{kd}$. Then for all $1 \leq i \leq d$ we have:
	\be
	%f(\mathbf{d}(\mathbf{e}_i);\ww)=\frac{\sum_{l=1}^{k}{\sigma((\aa^{S_l})^T\eee_i)}}{k} =\frac{1}{k} 
	f(\xx_i;\ww)=\frac{\sum_{l=1}^{k}{\sigma((\aa^{S_l})^T\eee_i)}}{k} =\frac{1}{k}  = y_i
	\ee
	and for all $j$:
	\be
	f(\xx_{d+j};\ww)=\frac{\sum_{l=1}^{k}{\sigma((\mathbf{a}^{S_l})^T(\sum_{i\in C_j}{\eee_i}))}}{k} = 0  = y_{d+j}
	\ee
	where the last equality follows since for all $l$ and $j$, $C_j \not \subseteq S_l$. Therefore there exists a $\ww$ for which the error in \eqref{eq:kopt} is zero and {\em{\kopt}} will return a weight vector with low risk.
	
	Conversely, assume that {\em{\kopt}} returned a $\ww\in\reals^{kd}$ with risk less than $\frac{1}{4k^5d}$ on $\cD_{X,Y}$ above. Denote by $\ww = (\ww_1,\ww_2,...,\ww_k)$, where $\ww_l \in \reals^d$.  We will show that this implies that there exist $k$ splitting sets. 
	\ignore{
		Note that generally in this case we are not guaranteed that $f$ is a \textit{$k$-non-overlap} network, however since we will show existence of splitting sets, this will follow from the previous argument.}
	For all $\xx',y'$ in the training set it holds that:\footnote{The LHS is true because for a non-negative random variable $X$, $E[X]\geq p(x)x$ for all $x$, and in our case $p(x)\geq {1\over {kd}}$.}
	\be
	\frac{\left(f(\xx';\ww) -y'\right)^2}{kd} \leq \mathbb{E}_{\cD_{X,Y}}[(f(\xx;\ww)-y)^2] < \frac{1}{4k^5d} \nonumber
	\ee
	\\	
	This implies that for all $i$ and $j$, 
	\be
	|f(\dd(\eee_i);\ww) - \frac{1}{k}| < \frac{1}{2k^2}  \ \ , \ \ 
	|f(\dd(\sum_{i\in C_j}{\eee_i});\ww)| < \frac{1}{2k^2}
	\ee
	Define sets $S_l = \{i \mid \ww_l^T\eee_i > \frac{1}{2k}\}$ for $1 \leq l \leq k$ and WLOG assume they are disjoint by arbitrarily assigning points that belong to more than one set, to one of the sets they belong to. We will next show that these $S_l$ are splitting. Namely, it holds that $\bigcup_l{S_l}=S$ and no subset $C_j$ is a subset of some $S_l$.
	
	Since $f(\dd(\eee_i);\ww) = \frac{\sum_{l=1}^{k}{\sigma(\ww_l^T\eee_i)}}{k} > \frac{1}{k} - \frac{1}{2k^2} > \frac{1}{2k}$ for all $i$, it follows that for each $i \in S$ there exists $1 \leq l \leq k$ such that $\ww_l^T\eee_i > \frac{1}{2k}$. Therefore, by the definition of $S_l$ we deduce that $\bigcup_l{S_l}=S$. To show the second property, assume by contradiction that for some $j$ and $m$, $C_j \subseteq S_m$. Then $\ww_m^T(\sum_{i\in C_j}{\eee_i}) > \frac{|C_j|}{2k}$, which implies that $f(\dd(\sum_{i\in C_j}{\eee_i});\ww) = \frac{\sum_{l=1}^{k}{\sigma(\ww_l^T(\sum_{i\in C_j}{\eee_i}))}}{k} > \frac{|C_j|}{2k^2} \geq \frac{1}{2k^2}$, a contradiction. This concludes our proof. 
\end{proof}

To conclude, we have shown that {\ourarch} are hard to learn if one does not make any assumptions about the training data. In fact we have shown that finding a $\ww$ with loss at most ${1\over 4 k^5 d}$ is hard. In the next section, we show that certain distributional assumptions make the problem tractable.

\section{{\ourarch} can be Learned for Gaussian Inputs}
%Learning the Non-Overlap Network With Gradient Descent Under Gaussian Data}
\label{non_overlapping_gaussian}

In this section we assume that the input features $\xx$ are generated via a Gaussian distribution $\cG$, as in \secref{preliminaries}. We will show that in this case, gradient descent will converge with high probability to the global optimum of $\ell(\ww)$  (\eqref{non_overlap_loss_eq}) in polynomial time.

\ignore{
	Following the result of the previous section, we see that we need to restrict the definition of the problem \textit{k-nonoverlap-training-approximation} in order to show tractability of learning a \textit{non-overlap} network. Indeed, here we restrict the distribution $\mathcal{D}$ of the training data - we now assume that the entries of the data are standard i.i.d. gaussian random variables. In this case we show that a randomly initialized gradient descent will converge to the unique global minimum of the loss function. Moreover, we show that with high probability gradient descent converges to a point close to the global minimum in polynomial number of steps. Consequently, this implies that gradient descent solves the \textit{k-nonoverlap-training-approximation} problem under the Gaussian assumption on $\mathcal{D}$. We only sketch proof ideas in this section, the full details can be found in the supplementary material.
}

In order to analyze convergence of gradient descent on $\ell$, we need a characterization of all the critical and non-differentiable points. We show that $\ell$ has a non-differentiable point and a degenerate saddle point.\footnote{A saddle point is degenerate if the Hessian at the point has only non-negative eigenvalues and at least one zero eigenvalue.} Therefore, recent methods for showing global convergence of gradient-based optimizers on non-convex objectives \citep{LeeSJR16,ge2015escaping} cannot be used in our case, because they assume all saddles are strict \footnote{A saddle point is strict if the Hessian at the point has at least one negative eigenvalue.} and the objective function is continuously differentiable everywhere.

The characterization is given in the following lemma. The proof relies on the fact that $\ell(\mathbf{w})$ depends only on $\left\|\ww\right\|$,$\left\|\ww^*\right\|$ and $\theta_{\ww,\ww^*}$, and therefore w.l.o.g. it can be assumed that  $\ww^*$ lies on one of the axes. Then by a symmetry argument, in order to prove properties of the gradient and the Hessian, it suffices to calculate partial derivatives with respect to at most three variables.
\begin{lem}
	\label{LossPropertiesLem_conv}
	Let $\ell(\ww)$ be defined as in \eqref{non_overlap_loss_eq}. Then the following holds:
	\begin{enumerate}
		\item $\ell(\ww)$ is differentiable if and only if $\ww \neq \mathbf{0}$.
		\item For $k>1$, $\ell(\ww)$ has three critical points: 
		\begin{enumerate}
			\item  A local maximum at $\ww = \mathbf{0}$.
			\item  A unique global minimum at $\ww=\ww^*$. 
			\item  A degenerate saddle point at  $\ww = -(\frac{k^2-k}{k^2 + (\pi -1) k})\ww^*$.
		\end{enumerate}	
		
		For $k=1$, $\ww = \mathbf{0}$ is not a local maximum and the unique global minimum $\ww^*$ is the only differentiable critical point.
		
	\end{enumerate}
	
\end{lem}

\ignore{We note that recent methods for showing global convergence of non-convex gradient descent \citep{ge2015escaping,LeeSJR16}  cannot be used in our case because our $\ell(\ww)$ is not differentiable as required by these results. The dependence of the convergence rate on $\epsilon$ is similar to \citep{ge2015escaping}), but since our focus here was to show poly-time in ${1\over\epsilon}$ we leave further improvements to future work.}

We next consider a simple gradient descent update rule for minimizing $\ell(\ww)$ and analyze its convergence. Let $\lambda > 0$ denote the step size. Then the update at iteration $t$ is simply:
\be
\ww_{t+1} = \ww_t - \lambda\nabla \ell(\ww_t)
\label{eq:gd_updates}
\ee

Our main result, stated formally below, is that the above update is guaranteed to converge to an $\epsilon$ accurate solution after $O({1\over\epsilon^2})$ iterations. We note that the dependence of the convergence rate on $\epsilon$ is similar to standard results on convergence of gradient descent to stationary points \citep[e.g., see discussion in][]{allen2016variance}.

\begin{thm}
	\label{non_overlap_poly_conv}
	Assume $\left \|\ww^* \right\| = 1$.\footnote{Assumed for simplicity, otherwise $\left\|\ww^*\right\|$ is a constant factor.} For any $\delta > 0$ and $0 < \epsilon < \frac{\delta\sin\pi\delta}{k}$, there exists $0 < \lambda < 1$ \footnote{$\lambda$ can be found explicitly.} such that with probability at least $1-\delta$, gradient descent initialized randomly from the unit sphere with learning rate $\lambda$ will get to a point $\ww$ such that $\ell(\ww) \leq O(\epsilon)$ \footnote{$O(\cdot)$ hides a linear factor in $d$.} in $O(\frac{1}{\epsilon^2})$ iterations.
\end{thm}
The complete proof is provided in Appendix \ref{supp_non_overlapping_gaussian}. Here we provide a high level overview. In particular, we first explain why gradient descent will stay away from the two {\em bad} points mentioned in Lemma \ref{LossPropertiesLem_conv}.

First we note that the gradient of $\ell(\ww)$ at $\ww_t$ is given by: 
\be
\nabla \ell(\ww_t) = -c_1(\ww_t,\ww^*)\ww_t - c_2(\ww_t,\ww^*)\ww^* ~,
\ee
where $c_1$ and $c_2$ are two functions such that $c_1 \geq -1$ and $c_2\geq 0$. Thus the gradient is a sum of a vector in the direction of $\ww_t$ and a vector in the direction of $\ww^*$. At iteration $t+1$ we have:

\be
\ww_{t+1} = (1+\lambda c_1(\ww_t,\ww^*))\ww_t + \lambda c_2(\ww_t,\ww^*)\ww^*
\label{eq:w_t_plus_one}
\ee
It follows that for $\lambda < 1$ the angle between $\ww_t$ and $\ww^*$ will decrease in each iteration. Therefore, if $\ww_0$ has an angle with $\ww^*$ that is not $\pi$, we will never converge to the saddle point in Lemma \ref{LossPropertiesLem_conv}.  

Next, assuming $\|\ww_0\|>0$ and that the angle between $\ww_0$ and $\ww^*$ is at most $(1-\delta)\pi$ (which occurs with probability $1-\delta$), it can be shown that the norm of $\ww_t$ is always bounded away from zero by a constant $M = \tilde{\Omega}(1)$.\footnote{\label{foot:tilde}$\tilde{\Omega}$ and $\tilde{O}$ hide factors of $\left\|\ww^*\right\|$, $\theta_{\ww_0,\ww^*}$, $k$ and $\delta$.} The proof is quite technical and follows from the fact that $\ww = \mathbf{0}$ is a local maximum.\footnote{The proof holds even for $k=1$ where $\ww = \mathbf{0}$ is not a local maximum.}

The fact that $\ww_t$ stays away from the {\em problematic} points allows us to show that $\ell(\ww)$ has a Lipschitz continuous gradient on the line between $\ww_{t}$ and $\ww_{t+1}$, with constant $L = \tilde{O}(1)$.\textsuperscript{\ref{foot:tilde}} By standard optimization analysis \citep{nesterov2004introductory} it follows that after $T=O({1\over \epsilon^2})$ iterations we will have
$\|\nabla l(\ww_{t})\|\leq O(\epsilon)$ for some $0 \leq t \leq T$. This in turn can be used to show that $\ww_t$ is $O(\sqrt{\epsilon})$-close to $\ww^*$. Finally, since $\ell(\ww) \leq d{\left\|\ww-\ww^*\right\|}^2$, it follows that $\ww_t$ approximates the global minimum to within $O(\epsilon)$ accuracy.

Theorem \ref{non_overlap_poly_conv} implies that gradient descent converges to a point $\ww$ such that $\ell(\ww) \leq \frac{1}{d^2}$ in time $O(poly(d))$ where $d$ is the input dimension.\footnote{Note that the complexity of a gradient descent iteration is polynomial in $d$.} The following corollary thus follows. 
\begin{cor}
	Gradient descent solves the {\kopt} problem under the Gaussian assumption on $\mathcal{D}$ with high probability and in polynomial time.
\end{cor}

\section{Empirical Illustration of Tractability Gap}
\label{nonoverlap_experiments}

The results in the previous sections showed that {\ourarch} optimization is hard in the general case, but tractable for Gaussian inputs. Here we empirically demonstrate both the easy and hard cases. The training data for the two cases
will be generated by using the same $\ww^*$ but different distributions over $\xx$.

To generate the ``hard'' case, we begin with a set splitting problem. In particular, we consider a set $S$ with $40$ elements and a collection ${\cal C}$ of  $760$ subsets of $S$, each of size $20$. We choose $C_j$ such that there exists subsets $S_1$,$S_2$ that split the subsets $C_j$. We use the reduction in \secref{hardness_section} to convert this into a {\ourarch} optimization problem. This results in a training set of size $800$.  

\begin{wrapfigure}{r}{6.0cm}
	\includegraphics[width=6.0cm]{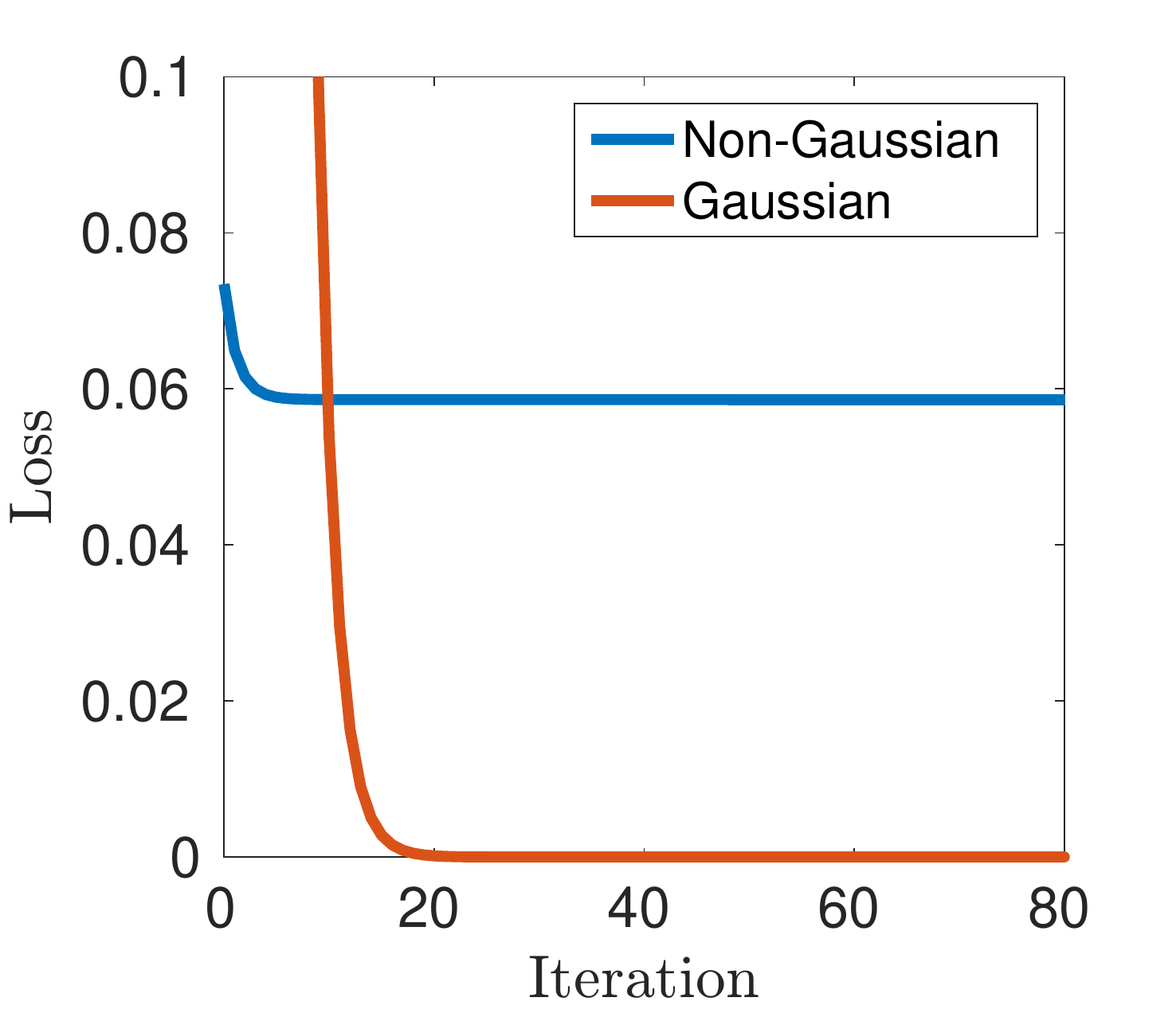}		
	\caption{Training loss of Adagrad on the Gaussian and Non-Gaussian datasets. See \secref{nonoverlap_experiments} for details.}
	\label{non_overlap_exp}
\end{wrapfigure}

Since we know the $\ww^*$ that solves the set splitting problem, we can use it to label data from a different distribution. Motivated by \secref{non_overlapping_gaussian} we use a Gaussian distribution $\cG$ as defined earlier and generate a training set of the same size (namely $800$) and labels given by the \textit{no-overlap} network with weight $\ww^*$.

\ignore{In both, $\ww^*$ is used to label the features, but the feature distribution is different, and corresponds to the two following cases: 
	
	\begin{itemize}
		\item Non Gaussian - Here the features $\xx$ are sampled from a uniform distribution over the XX features used in the set splitting reduction
		\item Gaussian - Here the features $\xx$ are sampled from a Gaussian distribution $\cG$. 
	\end{itemize}} 
	
	For these two learning problems we used AdaGrad \citep{duchi2011adaptive} to optimize the \textit{empirical} risk (plain gradient descent also converges, but AdaGrad requires less tuning of step size). For both datasets we used a random normal initializer and for each we chose the best performing learning rate schedule. The training error for each setting as a function of the number of epochs is shown in \figref{non_overlap_exp}. It is clear that in the non-Gaussian case, AdaGrad gets trapped at a sub-optimal point, whereas the Gaussian case is solved optimally.\footnote{We note that the value of $0.06$ attained by the non-Gaussian case is quite high, since the zero weight vector in this case has loss of order $0.1$.} In the Gaussian case AdaGrad converged to $\ww^*$. Therefore, given the Gaussian dataset we were able to recover the true weight vector $\ww^*$, whereas given the data constructed via the reduction we were not, even though both datasets were of the same size. We conclude that these empirical findings are in line with our theoretical results.

	\ignore{
		In this section we show that our theory of \textit{non-overlap} networks predicts cases where learning a neural network is hard and when it is easy. The hardness result in Section~\ref{hardness_section} suggests that given the data constructed in the reduction, it will not be possible to efficiently learn the weights of the filter of the \textit{non-overlap} network that produced the labels (this is the filter constructed in the first part of the proof, i.e., when there exists splitting sets). However, given a different data set of the \textit{same} size and labeled by the \textit{same} network, will it then be possible to learn the network parameters? 
		
		Inspired by the results in Section~\ref{non_overlapping_gaussian} we considered a dataset of size 800 with standard i.i.d. gaussian entries and a reduction dataset of size 800 in which we created 40 points that correspond to elements of a set $S$ and 760 points that correspond to subsets $c_j$. We chose $S$ and $c_j$ such that there exists disjoint sets $S_1,S_2 \subseteq S$ that split the subsets $c_j$. Both datasets are labeled by the \textit{non-overlap} network with a filter that corresponds to the splitting sets $S_1$ and $S_2$ (See Section~\ref{hardness_section}). We then compared the performance of AdaGrad when optimizing the average squared loss of a \textit{non-overlap} network given the reduction data and given the gaussian data. For both datasets we used a random normal initializer for the weights and for each we chose the best performing learning rate schedule.
	}
	
\section{Networks with Overlapping Filters}
\label{sec:nets_with_overlap}

Thus far we showed that the non-overlapping case becomes tractable under Gaussian inputs. A natural question is then what happens when overlaps are allowed (namely, the stride is
smaller than the filter size). Will gradient descent still find a global optimum? Here we show that this is in fact {\em not} the case, and that with probability greater than $\frac{1}{4}$ gradient descent will get stuck in a sub-optimal region. In \secref{overlapping_section} we analyze this setting for a two dimensional example and provide bounds on the level of suboptimality. In \secref{general_conv_experiments} we report on an empirical study of optimization 
for networks with overlapping filters. Our results suggest that by restarting gradient descent a constant number of times, it will converge to the global minimum with high probability.  Complete proofs of the results are provided in 
Appendix \ref{supp:overlapping_gaussian}. 

\ignore{
	In this section we study the general case under the gaussian data assumption, where we do not assume anything on the overlapping of the filters. In Section~\ref{overlapping_section} we show that when the filters are overlapping, the previous result for the non-overlapping case does not hold in general. Namely, under the assumptions of the previous section, there is a ConvNet with overlapping filters such that a randomly initialized gradient descent will get stuck in a sub-optimal region with probability greater than $\frac{1}{4}$. We further provide a tight lower bound on the loss of any point in this sub-optimal region. In Section~\ref{general_conv_experiments} we empirically study the tractability of optimization in the general setting. Our results suggest that by restarting gradient descent a constant number of times, it will converge to the global minimum with high probability.  
}

\subsection{Suboptimality of Gradient Descent for $\reals^2$}
\label{overlapping_section}

We consider an instance where there are $k = d - 1$ neurons and matrices $W,W^* \in \reals^{k \times d}$ correspond to an overlapping filter of size $2$ with stride $1$, i.e., for all $1 \leq i \leq k$ $\ww_i = (\mathbf{0}_{i-1},\ww,\mathbf{0}_{d-i-1})$, $\ww_i^* = (\mathbf{0}_{i-1},\ww^*,\mathbf{0}_{d-i-1})$ where $\mathbf{0}_{l} = (0,0,...,0) \in \reals^l$, $\ww = (w_1,w_2)$ is a vector of $2$ parameters and $\ww^* =(-w^*,w^*) \in \reals^2$, $w^* > 0$. Define the following vectors $\ww_r = (w_1,w_2,0)$, $\ww_l=(0,w_1,w_2)$, $\ww^*_r = (-w^*,w^*,0)$, $\ww^*_l = (0,-w^*,w^*)$ and denote by $\theta_{\ww,\vv}$ the angle between two vectors $\ww$ and $\vv$. 

One might wonder why the analysis of the overlapping case should be any different than the non-overlapping case. However, even for a filter of size two, as above, the loss function and consequently the gradient, are more complex in the overlapping case. Indeed, the loss function in this case is given by:

\begin{equation} 
\label{overlapping_loss}
\begin{split}
\ell(\ww) &= \alpha({\left\|\ww\right\|}^2 + {\left\|\ww^*\right\|}^2) - \beta g(\ww,\ww^*) \\ &+ (\beta-2)(g(
\ww_r,\ww_l) - g(\ww_l,\ww^*_r) \\&- g(\ww_r,\ww^*_l) + g(\ww^*_r,\ww^*_l)) - \gamma\left\|\ww\right\|\left\|\ww^*\right\|
\end{split}
\end{equation}

where $\alpha = \frac{1}{k^2}\big(\frac{k}{2} + \frac{k^2-3k+2}{2\pi}\big)$, $\beta = 2k$ and $\gamma=\frac{k^2-3k+2}{\pi}$.

\begin{wrapfigure}{r}{8.0cm}
	\includegraphics[width=8.0cm]{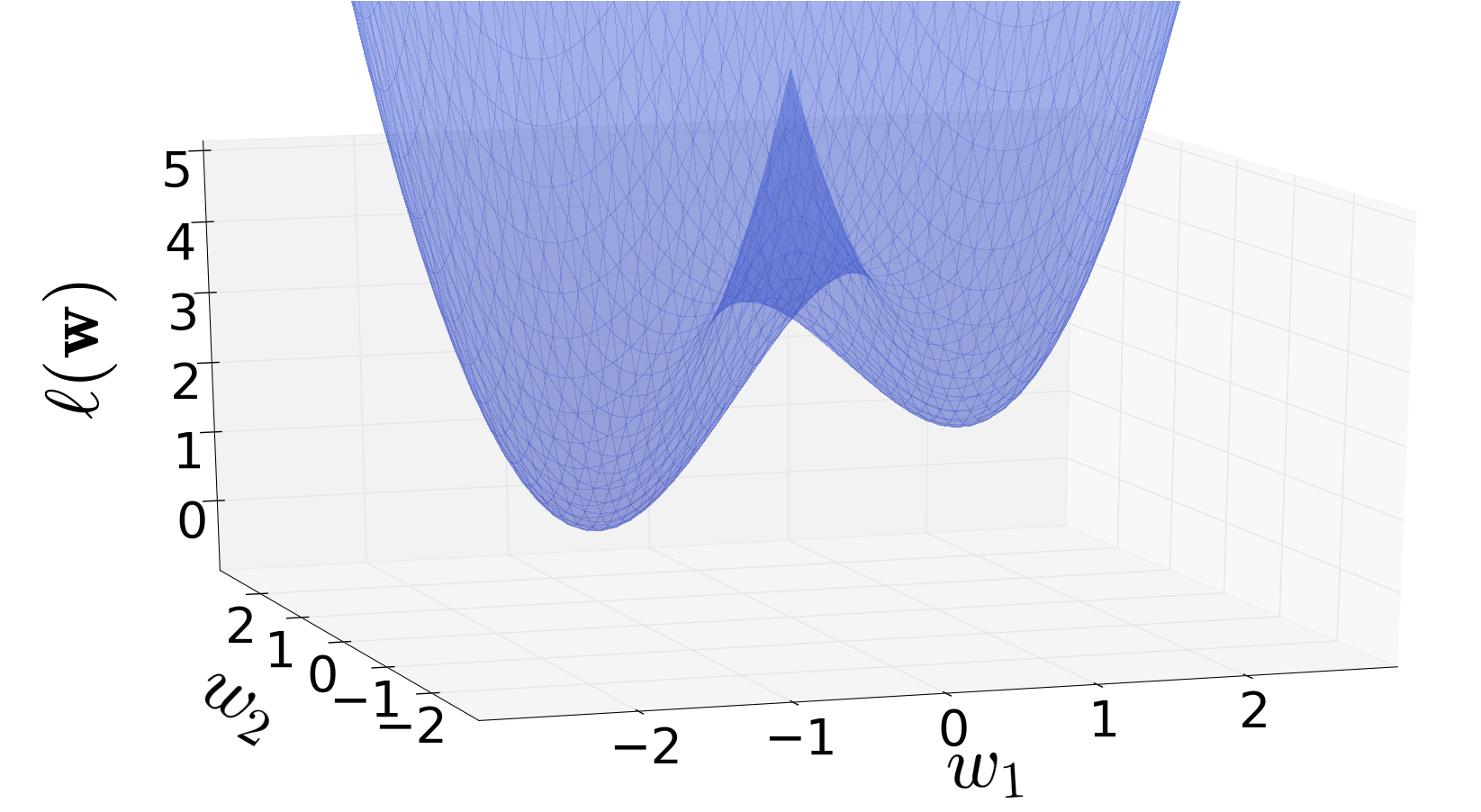}		
	\caption{The population risk for a network with overlapping filters, with a two dimensional filter $\ww^*=[-1,1]$, $k=4$, $d=5$, and Gaussian inputs. 
	} 
	\label{overlap_3d}
\end{wrapfigure} 

\ignore{
	\begin{equation} 
	\label{overlapping_loss}
	\begin{split}\label{key}
	l(\ww) &= \frac{1}{n^2}\Bigg[\big(\frac{n}{2} + \frac{n^2-3n+2}{2\pi}\big){\left\|\ww\right\|}^2 + 2(n-1)g(
	\ww_r,\ww_l) \\ &- \frac{\left\|\ww\right\|\left\|\ww^*\right\|}{\pi}\big(n^2-3n+2\big) 
	- 2ng(\ww,\ww^*) \\ &- 2(n-1)g(\ww_l,\ww^*_r) - 2(n-1)g(\ww_r,\ww^*_l) \\
	&+ \big(\frac{n}{2} + \frac{n^2-3n+2}{2\pi}\big){\left\|\ww^*\right\|}^2 + 2(n-1)g(\ww^*_r,\ww^*_l)\Bigg]
	\end{split}
	\end{equation}
}

Compared to the objective in \eqref{non_overlap_loss_eq} which depends only on $\left\|\ww\right\|$, $\left\|\ww\right\|$ and $\theta_{\ww,\ww^*}$, we see that the objective in \eqref{overlapping_loss} has new terms such as $g(\ww_r,\ww_l^*)$ which has a more complicated dependence on the weight vectors $\ww^*$ and $\ww$. This does not only have implications on the analysis, but also on the geometric properties of the loss function and the dynamics of gradient descent. In particular, in Figure~\ref{overlap_3d} we see that the objective has a large sub-optimal region which is not the case when the filters are non-overlapping.

As in the previous section we consider gradient descent updates as in \eqref{eq:gd_updates}. The following Proposition shows that if $\ww$ is initialized in the interior of the fourth quadrant of $\reals^2$, then it will stay there for all remaining iterations. The proof is a straightforward inspection of the components of the gradient, and is provided in the supplementary.

\begin{prop}
	\label{fourth_quad_prop}
	For any $\lambda \in (0,\frac{1}{3})$, if $\ww_t$ is in the interior of the fourth quadrant of $\reals^2$ then so is $\ww_{t+1}$.
\end{prop}

Note that in our example the global optimum $\ww^*$ is in the second quadrant (it's easy to show that it is also unique). Hence, if initialized at the fourth quadrant, gradient descent will remain in a sub-optimal region. The sub-optimality can be clearly seen in Figure~\ref{overlap_3d}. In the proposition below we formalize this observation by giving a tight lower bound on the values of $\ell(\ww)$ for $\ww$ in the fourth quadrant. Specifically, we show that the sub-optimality scales with $O(\frac{1}{k^2})$. The proof idea is to express all angles between all the vectors that appear in \eqref{overlapping_loss} via a single angle parameter $\theta$ between $\ww$ in the fourth quadrant and the positive $x$-axis. Then it is possible to prove the relatively simpler one dimensional inequality that depends on $\theta$.

\begin{prop}
	\label{tight_bound_prop}
	Let $h(k) = \frac{k^2-3k+2}{\pi} + \frac{\sqrt{3}(k-1)}{\pi} + \frac{2(k-1)}{3}$, then for all $\ww$ in the fourth quadrant $l(\ww) \geq \frac{2h(k)+1}{k^2(2h(k)+2)}{\left\|\ww^*\right\|}^2$ and this lower bound is attained by $\tilde{\ww}=-\frac{h(k)}{h(k)+1}\ww^*$.
\end{prop}

The above two propositions result in the following characterization of the sub-optimality of gradient descent for $\ww\in\reals^2$ and overlapping filters. 

\begin{thm}
	Define $h(k)$ as in Proposition \ref{tight_bound_prop}. Then with probability $\geq \frac{1}{4}$, a randomly initialized gradient descent with learning rate $\lambda \in (0,\frac{1}{3})$ will get stuck in a sub-optimal region, where each point in this region has loss at least $\frac{2h(k)+1}{k^2(2h(k)+2)}{\left\|\ww^*\right\|}^2$ and this bound is tight.
\end{thm}

\subsection{Empirical study of Gradient Descent for $m>2$}
\label{general_conv_experiments}
In \secref{overlapping_section} we showed that already for $m=2$, networks with $\ww\in\reals^m$ and filter overlaps exhibit more complex behavior than those without overlap. This leaves open the question of what happens in the general 
case under the Gaussian assumption, for various values of $d,m$ and overlaps.  We leave the theoretical analysis of this question to future work, but here report on empirical findings that hint at what the solution should look like.

\ignore{
	One intriguing open problem that follows our theoretical results is whether there is a convolutional neural network such that gradient descent will converge to the global minimum with exponentially small probability. Note that in our setting there is a unique global minimum. The proof is given in the supplementary material.
	\begin{prop}
		\label{unique_min_prop}
		Under our model assumptions, the loss function of a convolutional neural network has a unique global minimum.
	\end{prop}
	
}
We experimented with a range of $d,m$ and overlap values (see Appendix \ref{supp:experimental_setup} for details of the experimental setup). For each value of $d$, $m$ and overlap we sampled $90$ values of $\ww^*$ from various uniform input distributions with different supports and several pre-defined deterministic values. This resulted in more than 1200 different sampled $\ww^*$. For each such $\ww^*$ we ran gradient descent multiple times, each initialized randomly from a different $\ww_0$. Using the results from these runs, we could estimate the probability of sampling a $\ww_0$ that would converge to the \textit{unique} global minimum. Viewed differently, this is the probability mass of the basin of attraction of the global optimum. We note that the uniqueness of the global minimum follows easily from equating the population risk (\eqref{population_risk}) to 0 and the full proof is deferred to Appendix \ref{supp:min_prop_section}.

%Through extensive experiments we suggest that the answer to the latter question is negative, namely, with a constant number of restarts, a randomly initialized gradient descent will converge to the global minimum with high probability. For various values of number of neurons, filter size, stride and ground truths we computed a confidence interval for the probability of converging to the global minimum with a randomly initialized gradient descent (see Section \ref{experimental_setup} for details of the experimental setup).
Our results are that across all values of $d,m$, overlap and $\ww^*$, the probability mass of the basin of attraction is at least $\frac{1}{17}$. The practical implication is that multiple restarts of gradient descent (in this case a few dozen) will find the global optimum with high probability.  We leave formal analysis of this intriguing fact for future work. 

%Our results are quite surprising: the minimum lower limit over all confidence intervals is $\frac{1}{11}$. This suggests that essentially non-convexity is not an issue for learning convolutional neural networks in our setting with gradient descent- with an order of $11$ repeated runs, gradient descent will eventually converge to the global minimum. It would be very interesting to prove this property theoretically.

\section{Discussion}
\label{sec:discussion} 
The key theoretical question in deep learning is why it succeeds in finding good models despite the non-convexity of the training loss. It is clear
that an answer must characterize specific settings where deep learning provably works. Despite considerable recent effort, such a case has not been shown. Here
we provide the first analysis of a non-linear architecture where gradient descent is globally optimal, for a certain input distribution, namely Gaussian. Thus our specific
characterization is both in terms of architecture (no-overlap networks, single hidden layer, and average pooling) and input distribution. We show that  
learning in no-overlap architectures is hard, so that some input distribution restriction is necessary for tractability. Note however, that it is certainly possible
that other, non-Gaussian, distributions also result in tractability. Some candidates would be sub-Gaussian and log-concave distributions.

Our derivation addressed the population risk, which for the Gaussian case can be calculated in closed form. In practice, one minimizes an empirical risk. Our experiments in \secref{nonoverlap_experiments} suggest that optimizing the empirical risk in the Gaussian case is tractable. It would be interesting to prove this formally. It is likely that measure concentration results can be used to get similar results to those we had for the population risk \citep[e.g., see][for use of such tools]{mei2016landscape,xu2016global}.

Convolution layers are among the basic building block of neural networks. Our work is among the first to analyze optimization for these. The architecture we study is similar in structure to convolutional networks, in the sense
of using parameter tying and pooling. However, most standard convolutional layers have overlap and use max pooling. In \secref{sec:nets_with_overlap} we provide initial results for the case of overlap, showing there is hope for proving 
optimality for gradient descent with random restarts. Analyzing max pooling would be very interesting and is left for future work. 

Finally, we note that distribution dependent tractability has been shown for intersection of halfspaces \citep{klivans2009baum}, which is a non-convolutional architecture. However, these results do not use gradient descent. It would be very interesting to use our techniques to try and understand gradient descent for the population risk in these settings.    

\bibliography{non_overlap}
\bibliographystyle{icml2017}

\appendix
%\renewcommand{\thesubsection}{\Alph{subsection}}
%\tableofcontents
\section{Proof of Lemma 3.2}%\ref{kernel_gradient}}
\label{prelim_supp}

First assume that $\theta_{\mathbf{u},\mathbf{v}} \neq 0,\pi$ . Then by straightforward calculation we have
\begin{equation}
\begin{split}
\frac{\partial g}{\partial u_i} &= \frac{1}{2\pi}\left \| \mathbf{v} \right \| \frac{u_i}{\left \| \mathbf{u} \right \|}\Bigg(\sqrt{1-\Big(\frac{\mathbf{u} \cdot \mathbf{v}}{\left \| \mathbf{u} \right \| \left \| \mathbf{v} \right \|}\Big)^2} + \Big(\pi - \arccos\Big(\frac{\mathbf{u} \cdot \mathbf{v}}{\left \| \mathbf{u} \right \| \left \| \mathbf{v} \right \|}\Big)\Big)\frac{\mathbf{u} \cdot \mathbf{v}}{\left \| \mathbf{u} \right \| \left \| \mathbf{v} \right \|} \Bigg) \\ &+ \frac{1}{2\pi}\left \| \mathbf{u} \right \| \left \| \mathbf{v} \right \| \Biggl( \Bigg(\ -\frac{\frac{\mathbf{u} \cdot \mathbf{v}}{\left \| \mathbf{u} \right \| \left \| \mathbf{v} \right \|}}{\sqrt{1-\Big(\frac{\mathbf{u} \cdot \mathbf{v}}{\left \| \mathbf{u} \right \| \left \| \mathbf{v} \right \|}\Big)^2}} \Bigg) \Bigg(\frac{v_i}{\left \| \mathbf{u} \right \| \left \| \mathbf{v} \right \|} -  \frac{u_i}{{\left \| \mathbf{u} \right \|}^2 } \frac{\mathbf{u} \cdot \mathbf{v}}{\left \| \mathbf{u} \right \| \left \| \mathbf{v} \right \|} \Bigg) \\ &+ \Bigg( \frac{\frac{\mathbf{u} \cdot \mathbf{v}}{\left \| \mathbf{u} \right \| \left \| \mathbf{v} \right \|}}{\sqrt{1-\big(\frac{\mathbf{u} \cdot \mathbf{v}}{\left \| \mathbf{u} \right \| \left \| \mathbf{v} \right \|}\big)^2}} \Bigg(\frac{v_i}{\left \| \mathbf{u} \right \| \left \| \mathbf{v} \right \|} -  \frac{u_i}{{\left \| \mathbf{u} \right \|}^2 } \frac{\mathbf{u} \cdot \mathbf{v}}{\left \| \mathbf{u} \right \| \left \| \mathbf{v} \right \|} \Bigg) \Bigg) \\ &+ \Big(\pi - \arccos\Big(\frac{\mathbf{u} \cdot \mathbf{v}}{\left \| \mathbf{u} \right \| \left \| \mathbf{v} \right \|}\Big)\Big) \Bigg(\frac{v_i}{\left \| \mathbf{u} \right \| \left \| \mathbf{v} \right \|} -  \frac{u_i}{{\left \| \mathbf{u} \right \|}^2 } \frac{\mathbf{u} \cdot \mathbf{v}}{\left \| \mathbf{u} \right \| \left \| \mathbf{v} \right \|} \Bigg) \Biggl) \\ &= \frac{1}{2\pi}\left \| \mathbf{v} \right \| \frac{u_i}{\left \| \mathbf{u} \right \|}\Bigg(\sqrt{1-\Big(\frac{\mathbf{u} \cdot \mathbf{v}}{\left \| \mathbf{u} \right \| \left \| \mathbf{v} \right \|}\Big)^2} + \Big(\pi - \arccos\Big(\frac{\mathbf{u} \cdot \mathbf{v}}{\left \| \mathbf{u} \right \| \left \| \mathbf{v} \right \|}\Big)\Big)\frac{\mathbf{u} \cdot \mathbf{v}}{\left \| \mathbf{u} \right \| \left \| \mathbf{v} \right \|} \Bigg) \\ &+ \frac{1}{2\pi}\left \| \mathbf{u} \right \| \left \| \mathbf{v} \right \| \Big(\pi - \arccos\Big(\frac{\mathbf{u} \cdot \mathbf{v}}{\left \| \mathbf{u} \right \| \left \| \mathbf{v} \right \|}\Big)\Big) \Bigg(\frac{v_i}{\left \| \mathbf{u} \right \| \left \| \mathbf{v} \right \|} -  \frac{u_i}{{\left \| \mathbf{u} \right \|}^2 } \frac{\mathbf{u} \cdot \mathbf{v}}{\left \| \mathbf{u} \right \| \left \| \mathbf{v} \right \|} \Bigg) \Biggl) \\ &= \frac{1}{2\pi}\left \| \mathbf{v} \right \| \frac{u_i}{\left \| \mathbf{u} \right \|} \sqrt{1-\Big(\frac{\mathbf{u} \cdot \mathbf{v}}{\left \| \mathbf{u} \right \| \left \| \mathbf{v} \right \|}\Big)^2} + \frac{1}{2\pi}\Big(\pi - \arccos\Big(\frac{\mathbf{u} \cdot \mathbf{v}}{\left \| \mathbf{u} \right \| \left \| \mathbf{v} \right \|}\Big) \Big) v_i \\ &= \frac{1}{2\pi}\left \| \mathbf{v} \right \| \frac{u_i}{\left \| \mathbf{u} \right \|} \sin \theta_{\uu,\vv} + \frac{1}{2\pi}\Big(\pi - \theta_{\uu,\vv} \Big) v_i
\end{split}
\end{equation}

Hence,

\begin{equation}
\label{gradRelu}
\frac{\partial g}{\partial \mathbf{u}} = \frac{1}{2\pi}\left \| \mathbf{v} \right \| \frac{\mathbf{u}}{\left \| \mathbf{u} \right \|} \sin \theta_{\uu,\vv} + \frac{1}{2\pi}\Big(\pi - \theta_{\uu,\vv} \Big) \mathbf{v}
\end{equation}

Now we assume that $\mathbf{u}$ is parallel to $\mathbf{v}$. We first show that $g$ is differentiable in this case. Without loss of generality we can assume that $\uu$ and $\vv$ lie on the $u_1$ axis. This follows since $g$ is a function of $\left\|\mathbf{u}\right\|$,  $\left\|\mathbf{v}\right\|$ and $\theta_{\mathbf{u},\mathbf{v}}$ and therefore $g(\cdot,\mathbf{v})$ has a directional derivative in direction $\mathbf{d}$ at $\mathbf{u}$ if and only if $g(\cdot,R\mathbf{v})$ has a directional derivative in direction $R\mathbf{d}$ at $R\mathbf{u}$ where $R$ is a rotation matrix. Hence $g(\cdot,\mathbf{v})$ is differentiable at $\mathbf{u}$ if and only if $g(\cdot,R\mathbf{v})$ is differentiable at $R\mathbf{u}$. Furthermore, if $\mathbf{v}$ and $\mathbf{u}$ are on the $u_1$ axis, then by symmetry the partial derivatives with respect to \textit{other} axes at $\mathbf{u}$ are all equal, hence we only need to consider the partial derivative with respect to the $u_1$ and $u_2$ axes.

Let $\vv=(1,0,...,0)$ and $\uu = (u,0,...,0)$ where $u \neq 0$. In order to show differentiability, we will prove that $g(\mathbf{u},\mathbf{v})$ has continuous partial derivatives at $\uu$ (by equality (\ref{gradRelu}) the partial derivatives are clearly continuous at points that are not on the $u_1$ axis. Define $\mathbf{u}_\epsilon = (u,\epsilon, 0,...,0)$. Then $$\frac{\partial g}{\partial u_2}(\mathbf{u},\mathbf{v}) = \lim_{\epsilon \to 0}{\frac{\frac{1}{2\pi}\left \| \mathbf{u}_\epsilon \right \| \left \| \mathbf{v} \right \|\Bigg(\sin \theta_{\uu_\epsilon,\vv} + \Big(\pi - \theta_{\uu_\epsilon,\vv} \Big) \cos \theta_{\uu_\epsilon,\vv} \Bigg) - g(\mathbf{u},\mathbf{v})}{\epsilon}}$$
By L'hopital's rule and the calculation of equality (\ref{gradRelu}) we get $$\frac{\partial g}{\partial u_2}(\mathbf{u},\mathbf{v}) = \lim_{\epsilon \to 0} { \frac{1}{2\pi}\left \| \mathbf{v} \right \| \frac{\epsilon}{\left \| \mathbf{u}_\epsilon \right \|} \sin \theta_\epsilon} = 0$$
Furthermore, by equality (\ref{gradRelu}) we see that $\lim_{\mathbf{u}' \to \mathbf{u}}{\frac{\partial g}{\partial u_2}(\mathbf{u}',\mathbf{v})} = 0$ since $\lim_{\mathbf{u}' \to \mathbf{u}}{\sin \theta_{\uu',\vv}} = 0$.

For a fixed $\theta_{\mathbf{u},\mathbf{v}}$ equal to $0$ or $\pi$, $\frac{\partial g}{\partial u_1}(\mathbf{u},\mathbf{v})$ is the same as $\frac{\partial g}{\partial \left\|\mathbf{u}\right\|}(\mathbf{u},\mathbf{v})$. Hence, 
$$\frac{\partial g}{\partial u_1}(\mathbf{u},\mathbf{v}) = \frac{1}{2\pi}\left \| \mathbf{v} \right \|\Bigg(\sin \theta_{\mathbf{u},\mathbf{v}} + \Big(\pi - \theta_{\mathbf{u},\mathbf{v}} \Big) \cos \theta_{\mathbf{u},\mathbf{v}} \Bigg) = \left\{\begin{matrix}\frac{1}{2} \text{  if $u > 0$} \\ \text{0  if $u < 0$} \end{matrix}\right.$$  
and the partial derivative is continuous since $$\lim_{\mathbf{u}' \to \mathbf{u}}{\frac{\partial g}{\partial u_1}(\mathbf{u}',\mathbf{v})} = \left\{\begin{matrix}\frac{1}{2} \text{  if $u > 0$} \\ \text{0  if $u < 0$} \end{matrix}\right.$$
\iffalse
Now let $\mathbf{w}_\epsilon = (0,y+\epsilon)$ then again by L'hopital's rule we have $$\frac{\partial g}{\partial y}(\mathbf{w},\mathbf{v}) = \lim_{\epsilon \to 0} { \frac{1}{2\pi}\left \| \mathbf{v} \right \| \frac{y + \epsilon}{\left \| \mathbf{w}_\epsilon \right \|} \sin \theta_\epsilok + \frac{1}{2\pi}(\pi - \theta_\epsilon)} = \left\{\begin{matrix}\frac{1}{2} \text{  if $y > 0$} \\ \text{0  if $y < 0$} \end{matrix}\right.$$
and $$\lim_{\mathbf{w}' \to \mathbf{w}}{\frac{\partial g}{\partial y}(\mathbf{w}',\mathbf{v})} = \left\{\begin{matrix}\frac{1}{2} \text{  if $y > 0$} \\ \text{0  if $y < 0$} \end{matrix}\right.$$
\fi

Finally, we see that for the case where $\mathbf{u}$ and $\mathbf{v}$ are parallel, the values we got for the partial derivatives coincide with equation \eqref{gradRelu}. This concludes the proof. 

\section{Proof of Proposition \ref{prop:set-splitting}}
\label{nonoverlap_hardness_supp}

We will prove the claim by induction on $k$. For the base case we will show that \textit{Set-Splitting-by-2-Sets} is NP-complete. We will prove this via a reduction from a variant of the 3-SAT problem with the restriction of equal number of variables and clauses, which we denote \textit{Equal-3SAT}. We will first prove that \textit{Equal-3SAT} is NP-complete.

\begin{lem}
	\textit{Equal-3SAT} is NP-complete.
\end{lem}
\begin{proof}
	This can be shown via a reduction from 3SAT. Given a formula $\phi$ with $n$ variables and $m$ clauses we can increase $n-m$ by $1$ by adding a new clause of the form ($x\vee y$) for new variables $x$ and $y$. Furthermore, we can decrease $n-m$ by $1$ by adding two new identical clauses of the form ($z$) for a new variable $z$. In each case the formula with the new clause(s) is satisfiable if and only if $\phi$ is. Therefore given a formula $\phi$ we can construct a new formula $\psi$ with equal number of variables and clauses such that $\phi$ is satisfiable if and only if $\psi$ is. 
\end{proof}

We will now give a reduction from \textit{Equal-3SAT} to \textit{Set-Splitting-by-2-Sets}.
\begin{lem}
	\textit{Set-Splitting-by-2-Sets} is NP-complete.
\end{lem}
\begin{proof}
	The following reduction is exactly the reduction from 3SAT to Splitting-Sets and we include it here for completeness. Let $\phi$ be a formula with set of variables $V$ and equal number of variables and clauses. We construct the sets $S$ and $\cal C$ as follows. Define $$S=\{\bar{x}\mid x\in V\} \cup V \cup \{n\}$$ where $\bar{x}$ is the negation of variable $x$ and $n$ is a new variable not in $V$. For each clause $c$ with set of variables or negations of variables $V_c$ that appear in the clause (for example, if $c=(\bar{x} \vee y)$ then $V_c = \{\bar{x},y\}$) construct a set $S_c = V_c \cup \{n\}$. Furthermore, for each variable $x \in V$ construct a set $S_x=\{x,\bar{x}\}$. Let $\cal C$ be the family of subsets $S_c$ and $S_x$ for all clauses $c$ and $x\in V$. Note that $|\cal C|$$ \leq |S|$ which is required by the definition of \textit{Set-Splitting-by-2-Sets}.
	
	Assume that $\phi$ is satisfiable and let $A$ be the satisfying assignment. Define $S_1 = \{x | A(x)=true\} \cup \{\bar{x} | A(x) = false\} $  and $S_2 = \{x | A(x)=false\} \cup \{\bar{x} | A(x) = true\} \cup \{n\}$. Note that $S_1 \cup S_2 = S$. Assume by contradiction that there exists a set $T \in \cal C$ such that $T \subseteq S_1$ or $T \subseteq S_2$. If $T \subseteq S_1$ then $T$ is not a set $S_c$ for some clause $c$ because $n \notin S_1$. However, by the construction of $S_1$ a variable and its negation cannot be in $S_1$. Hence $T \subseteq S_1$ is impossible. If $T \subseteq S_2$ then as in the previous claim $T$ cannot be a set $S_x$ for a variable $x$. Hence $T=S_c$ for some clause $c$. However, this implies that $A(c) = false$, a contradiction.
	
	Conversely, assume there exists splitting sets $S_1$ and $S_2$ and w.l.o.g. $n \in S_1$. We note that it follows that no variable $x$ and its negation $\bar{x}$ are both contained in one of the sets $S_1$ or $S_2$. Define the following assignment $A$ for $\phi$. For all $x \in V$ if $x \in S_1$ let $A(x) = false$, otherwise let $A(x) = true$. Note that $A$ is a well defined assignment. Assume by contradiction that there is a clause $c$ in $\phi$ which is not satisfiable. Since $S_2$ splits $S_c$ it follows that there exists a variable $x$ such that it or its negation $\bar{x}$ are in $S_2$ (recall that $n \in S_1$). If $x \in S_2$ then $A(x) = true$ and if $\bar{x} \in S_2$ then $A(\bar{x}) = true$ since $x \in S_1$. In both cases $c$ is satisfiable, a contradiction.
\end{proof}

This proves the base case. We will now prove the induction step by giving a reduction from \textit{Set-Splitting-by-k-Sets} to \textit{Set-Splitting-by-(k+1)-Sets}. Given $S=\{1,2,...,d\}$ and $\cal C$ $=\{C_j\}_j$ such that $|\cal C|$ $ \leq (k-1)d$, define $S'=\{1,2,...,d+1\}$ and $\cal C'$ $= \cal C$ $\cup \{D_j\}_j$ where $D_j = \{j,d+1\}$ for all $1\leq j \leq d$. Note that $|\cal C'|$ $\leq kd < k(d+1)$.
Assume that there are $S_1,...,S_k$ that split the sets in $\cal C$. Then if we define $S_{k+1} = \{d+1\}$, it follows that $\bigcup_{i=1}^{k+1}{S_i}=S$ and $S_1,...,S_k,S_{k+1}$ are disjoint and split the sets in $\cal C'$.

Conversely, assume that $S_1,...,S_k,S_{k+1}$ split the sets in $\cal C'$. Let w.l.o.g. $S_{k+1}$ be the set that contains $d+1$. Then for all $1 \leq j \leq d$ we have $D_j \not\subseteq S_{k+1}$. It follows that for all $1 \leq j \leq d$, $j \notin S_{k+1}$, or equivalently, $S_{k+1} = \{d+1\}$. Hence, $\bigcup_{i=1}^k{S_i}=S$ and $S_1,...,S_k$ are disjoint and split the sets in $\cal C$, as desired.

\section {Missing Proofs for Section \ref{non_overlapping_gaussian}} \label{supp_non_overlapping_gaussian}

\subsection{Proof of Lemma \ref{LossPropertiesLem_conv}}
\label{section_B1}
\begin{enumerate}
	\item For $\mathbf{w} \neq \mathbf{0}$, the claim follows from Lemma \ref{kernel_gradient}. As in the proof of Lemma \ref{kernel_gradient} we can assume w.l.o.g. that 
	$\mathbf{w}=(0,0,...,0)$ and $\mathbf{w}^*=(1,0,...,0)$. Let $f(\mathbf{w},\mathbf{w}^*) = 2kg(\mathbf{w},\mathbf{w}^*) + (k^2-k)\frac{\left \|\mathbf{w} \right \|\left \|\mathbf{w}^* \right \|}{\pi}$. It suffices to show that $\frac{\partial f}{\partial u_2}(\mathbf{w},\mathbf{w}^*)$ does not exist. Indeed, let $\mathbf{w}_\epsilon = (0,\epsilon,0,...,0)$ then by L'hopital's rule $$\lim_{\epsilon \to 0^+}{\frac{f(\mathbf{w}_\epsilon,\mathbf{w}^*) - f(\mathbf{w},\mathbf{w}^*)}{\epsilon}} = \lim_{\epsilon \to 0^+}{\frac{k}{\pi}\left \| \mathbf{w}^* \right \| \frac{\epsilon}{| \epsilon |} \sin \theta_{\ww_\epsilon,\ww^*}} + (k^2-k)\frac{\left \|\mathbf{w}^* \right \|}{\pi} = \frac{k}{\pi} + \frac{k^2-k}{\pi}$$
	and $$\lim_{\epsilon \to 0^-}{\frac{f(\mathbf{w}_\epsilon,\mathbf{w}^*) - f(\mathbf{w},\mathbf{w}^*)}{\epsilon}} = \lim_{\epsilon \to 0^-}{\frac{k}{\pi}\left \| \mathbf{w}^* \right \| \frac{\epsilon}{| \epsilon |} \sin \theta_{\ww_\epsilon,\ww^*}} - (k^2-k)\frac{\left \|\mathbf{w}^* \right \|}{\pi} = -\frac{k}{\pi} - \frac{k^2-k}{\pi}$$
	Hence the left and right partial derivatives with respect to variable $u_2$ are not equal, and thus $\frac{\partial f}{\partial u_2}(\mathbf{w},\mathbf{w}^*)$ does not exist.
	\item We first show that $\ww = \mathbf{0}$ is a local maximum if and only if $k>1$. Indeed, by considering the loss function as a function of the variable $x=\left\|\ww\right\|$, for any \textit{fixed} angle $\theta_{\ww,\ww^*}$ we get a quadratic function of the form $\ell(x)=ax^2-bx$, where $a > 0$ and $b \geq 0$. Since $f(\theta)=\sin\theta + (\pi - \theta)\cos\theta$ is a non-negative function for $0 \leq \theta \leq \pi$ and $f(\theta)=0$ if and only if $\theta=\pi$, it follows that $b=0$ if and only if $k=1$ \textit{and} $\theta_{\ww,\ww^*} = \pi$. Therefore if $k > 1$, then for all fixed angles $\theta_{\ww,\ww^*}$, the minimum of $\ell(x)$ is attained at $x > 0$, which implies that $\ww = \mathbf{0}$ is a local maximum. If $k = 1$ and $\theta_{\ww,\ww^*} = \pi$ the minimum of $\ell(x)$ is attained at $x=0$, and thus $\mathbf{w} = \mathbf{0}$ is not a local maximum in this case.  
	
	We will now find the other critical points of $\ell$. By Lemma %\ref{kernel_gradient} 
	3.2 we get 
	
	\begin{equation}
	\label{eq:loss_gradient}
	\begin{split}
	\nabla \ell(\mathbf{w}) &= \frac{1}{k^2} \Bigg[\big(k + \frac{k^2-k}{\pi}\big)\mathbf{w} - \frac{k}{\pi}\left \| \mathbf{w}^* \right \| \frac{\mathbf{w}}{\left \| \mathbf{w} \right \|} \sin \theta_{\ww,\ww^*} - \frac{k}{\pi}\Big(\pi - \theta_{\ww,\ww^*} \Big) \mathbf{w}^* - \frac{k^2-k}{\pi}\left \| \mathbf{w}^* \right \|\frac{\mathbf{w}}{\left \| \mathbf{w} \right \|}\Bigg] \\ &= 
	\frac{1}{k^2} \Bigg[\Bigg(k + \frac{k^2-k}{\pi} -  \frac{k\left \| \mathbf{w}^* \right \|}{\pi\left \| \mathbf{w} \right \|} \sin \theta_{\ww,\ww^*} - \frac{k^2-k}{\pi}\frac{\left \| \mathbf{w}^* \right \|}{\left \| \mathbf{w} \right \|} \Bigg) \mathbf{w} - \frac{k}{\pi}\Big(\pi - \theta_{\ww,\ww^*} \Big) \mathbf{w}^* \Bigg]
	\end{split}
	\end{equation}
	
	and assume it vanishes.
	
	Denote $\theta\triangleq\theta_{\ww,\ww^*}$. If $\theta = 0$ then let $\mathbf{w}=\alpha \mathbf{w}^*$ for some $\alpha > 0$. It follows that $$k + \frac{k^2-k}{\pi} - \frac{k^2-k}{\pi}\frac{1}{\alpha} - \frac{k}{\alpha} = 0$$ or equivalently $\alpha = 1$, and thus $\mathbf{w}=\mathbf{w}^*$.
	
	If $\theta = \pi$ then $\left \| \mathbf{w} \right \| = \frac{k^2-k}{k^2 + (\pi -1) k}\left \| \mathbf{w}^* \right \|$ and thus $\mathbf{w} = -(\frac{k^2-k}{k^2 + (\pi -1) k})\mathbf{w}^*$. By setting $\theta=\pi$ in the loss function, one can see that $\mathbf{w} = -(\frac{k^2-k}{k^2 + (\pi -1) k})\mathbf{w}^*$ is a one-dimensional local minimum, whereas by fixing $\left\|\mathbf{w}\right\|$ and decreasing $\theta$, the loss function decreases. It follows that $\mathbf{w} = -(\frac{k^2-k}{k^2 + (\pi -1) k})\mathbf{w}^*$ is a saddle point. If $\theta \neq 0,\pi$ then $\mathbf{w}$ and $\mathbf{w}^*$ are linearly independent and thus $\frac{k}{\pi}\Big(\pi - \theta \Big) = 0$ which is a contradiction.
	
	It remains to show that $\mathbf{u} = -\gamma(k)\mathbf{w}^*$ where $\gamma(k) = \frac{k^2-k}{k^2 + (\pi -1) k}$ is a degenerate saddle point. We will show that the Hessian at $\bs{u}$ denoted by $\nabla^2\ell(\mathbf{u})$, has only nonnegative eigenvalues and at least one zero eigenvalue. Let $\tilde{\ell}(\mathbf{w}) \triangleq \ell(\mathbf{w},R\mathbf{w}^*)$, where the second entry denotes the ground truth weight vector and $R$ is a rotation matrix. 
	Denote by $f_{\mathbf{d}_1,\mathbf{d}_2}$ the second directional derivative of a function $f$ in directions $\mathbf{d}_1$ and $\mathbf{d}_2$. Similarly to the proof of Lemma \ref{kernel_gradient}, since $\ell$ depends only on $\left\|\mathbf{w}\right\|$, $\left\|\mathbf{w}^*\right\|$ and $\theta_{\mathbf{w},\mathbf{w}^*}$, we notice that $$\ell_{\mathbf{d}_1,\mathbf{d}_2}(\mathbf{w}) = \tilde{\ell}_{R\mathbf{d}_1,R\mathbf{d}_2}(R\mathbf{w})$$ or equivalently $$\mathbf{d}_1^T\nabla^2\ell(\mathbf{w})\mathbf{d}_2 =  (R\mathbf{d}_1)^T\nabla^2\tilde{\ell}(R\mathbf{w})R\mathbf{d}_2 = \mathbf{d}_1^TR^T\nabla^2\tilde{\ell}(R\mathbf{w})R\mathbf{d}_2$$
	
	for any $\mathbf{w}$ and directions $\mathbf{d}_1$ and $\mathbf{d}_2$. It follows that $$\nabla^2\ell(\mathbf{w}) = R^T\nabla^2\tilde{\ell}(R\mathbf{w})R$$ for all $\mathbf{w}$. Since $R$ is an orthogonal matrix, we have that $\nabla^2\ell(\mathbf{w})$ and $\nabla^2\tilde{\ell}(R\mathbf{w})$ are similar matrices and thus have the same eigenvalues. Therefore, we can w.l.o.g. rotate $\mathbf{w}^*$ such that it will be on the $w_1$ axis.
	
By symmetry we have $$\frac{\partial \ell}{\partial w_1 \partial w_i}(\mathbf{u}) = \frac{\partial \ell}{\partial w_1 \partial w_j}(\mathbf{u}),\,\,\frac{\partial \ell}{\partial w_i \partial w_1}(\mathbf{u}) = \frac{\partial \ell}{\partial w_j \partial w_1}(\mathbf{u})$$ and $$\frac{\partial \ell}{\partial w_i^2}(\mathbf{u}) = \frac{\partial \ell}{\partial w_j^2}(\mathbf{u}),\,\,\frac{\partial \ell}{\partial w_i \partial w_j}(\mathbf{u}) = \frac{\partial \ell}{\partial w_s \partial w_t}(\mathbf{u})$$ 
	for $i\neq j,s \neq t$ such that $i,j,s,t \neq 1$. It follows that we only need to consider second partial derivatives with respect to 3 axes $w_1$,$w_2$ and $w_3$. Denote $\mathbf{u}_{\epsilon} = (-\gamma(k),\epsilon,0,...,0)$ and $\ww^*=(1,0,...,0)$ and $\beta(k) = \frac{k^2-k}{\pi}$ and note that $\gamma(k) = \frac{\beta(k)}{\beta(k)+k}$. Then by equation \eqref{eq:loss_gradient} we have 
	\begin{equation}
	\begin{split}
	\frac{\partial \ell}{\partial w_2^2}(\mathbf{u}) &= \lim_{\epsilon \to 0}{\frac{\nabla \ell (\mathbf{u}_{\epsilon})_x - \nabla\ell(\mathbf{u})_x}{\epsilon}} \\ &= \lim_{\epsilon \to 0}{\frac{\frac{1}{k^2} \Bigg[\big(k + \beta(k)\big)\epsilon - \frac{k}{\pi}\left \| \mathbf{w}^* \right \| \frac{\epsilon}{\left \| \mathbf{u}_{\epsilon} \right \|} \sin \theta_{\mathbf{u}_{\epsilon},\mathbf{w}^*} - \beta(k)\left \| \mathbf{w}^* \right \|\frac{\epsilon}{\left \| \mathbf{u}_{\epsilon} \right \|}\Bigg]}{\epsilon}} \\ &=
	\frac{1}{k^2}\big(k+\beta(k)-\frac{\beta(k)}{\gamma(k)}\big) = 0
	\end{split}
	\end{equation}
	
	Furthermore,
	
	\begin{equation}
	\begin{split}
	\frac{\partial \ell}{\partial w_1 \partial w_2}(\mathbf{u}) &= \lim_{\epsilon \to 0}{\frac{\nabla \ell (\mathbf{u}_{\epsilon})_y - \nabla\ell(\mathbf{u})_y}{\epsilon}} \\ &= \lim_{\epsilon \to 0}{\frac{\frac{1}{k^2} \Bigg[-\big(k + \beta(k)\big)\gamma(k) + \frac{k}{\pi}\left \| \mathbf{w}^* \right \| \frac{\gamma(k)}{\left \| \mathbf{u}_{\epsilon} \right \|} \sin \theta_{\mathbf{u}_{\epsilon},\mathbf{w}^*} + \beta(k)\left \| \mathbf{w}^* \right \|\frac{\gamma(k)}{\left \| \mathbf{u}_{\epsilon} \right \|} - \frac{k}{\pi}(\pi - \theta_{\mathbf{u}_{\epsilon},\mathbf{w}^*})\Bigg]}{\epsilon}}
	\end{split}
	\end{equation}
	
	%where $\theta_{\mathbf{w},\mathbf{w}^*}(\mathbf{w}) = \arccos\Big(\frac{\mathbf{w}\cdot\mathbf{w}^*}{\left\|\mathbf{w}\right\|\left\|\mathbf{w}^*\right\|}\Big)$. 
	
	%Since $$\frac{\partial \theta_{\mathbf{w},\mathbf{w}^*}}{\partial \mathbf{w}}(\mathbf{w}) = - \frac{1}{\sqrt{1-\big(\frac{\mathbf{w} \cdot \mathbf{w}^*}{\left \| \mathbf{w} \right \| \left \| \mathbf{w}^* \right \|}\big)^2}} \Bigg(\frac{\mathbf{w}^*}{\left \| \mathbf{w} \right \| \left \| \mathbf{w}^* \right \|} -  \frac{\mathbf{w}}{{\left \| \mathbf{w} \right \|}^2 } \frac{\mathbf{w} \cdot \mathbf{w}^*}{\left \| \mathbf{w} \right \| \left \| \mathbf{w}^* \right \|} \Bigg) $$ 
	where $\theta_{\mathbf{u}_{\epsilon},\mathbf{w}^*} = \arccos(\frac{-\gamma(k)}{\sqrt{\epsilon^2 + \gamma^2(k)}})$.	
	%Hence,
	%$$\frac{\partial \theta_{\mathbf{u}_{\epsilon},\mathbf{w}^*}}{\partial x}(\mathbf{u}_{\epsilon}) =  \frac{\epsilon\cos\theta_{\mathbf{u}_{\epsilon},\mathbf{w}^*}}{{\left \| \mathbf{u}_{\epsilon} \right \|}^2\sin\theta_{\mathbf{u}_{\epsilon},\mathbf{w}^*}} $$
	
	By L'Hopital's rule we have
	
	\begin{equation}
	\begin{split}
	\frac{\partial \ell}{\partial w_1 \partial w_2}(\mathbf{u}) &= \lim_{\epsilon \to 0}{ -\frac{\gamma(k)\epsilon\sin\theta_{\mathbf{u}_{\epsilon},\mathbf{w}^*}}{\pi k{\left\|\mathbf{u}_\epsilon\right\|}^3} + \frac{\gamma(k)\cos\theta_{\mathbf{u}_{\epsilon},\mathbf{w}^*}\frac{\partial \theta_{\mathbf{u}_{\epsilon},\mathbf{w}^*}}{\partial w_2}}{\pi k \left\|\mathbf{u}_\epsilon\right\|}} -\frac{\beta(k)\gamma(k)\epsilon}{{\left\|\mathbf{u}_\epsilon\right\|}^3} + \frac{\frac{\partial \theta_{\mathbf{u}_{\epsilon},\mathbf{w}^*}}{\partial w_2}}{\pi k}\\ &= \frac{1}{\pi k} \lim_{\epsilon \to 0}{\frac{\partial \theta_{\mathbf{u}_{\epsilon},\mathbf{w}^*}}{\partial w_2}\Big(\frac{\gamma(k)\cos\theta_{\mathbf{u}_{\epsilon},\mathbf{w}^*}}{\left\|\mathbf{u}_\epsilon\right\|} + 1\Big)}
	\end{split}
	\end{equation}
	
	Since $$\frac{\partial \theta_{\mathbf{u}_{\epsilon},\mathbf{w}^*}}{\partial w_2}(\mathbf{u}_{\epsilon}) = -\frac{1}{\frac{|\epsilon|}{\sqrt{\epsilon^2+\gamma^2(k)}}} \frac{\epsilon\gamma(k)}{(\epsilon^2+\gamma^2(k))^{\frac{3}{2}}} = -\frac{\epsilon\gamma(k)}{(\epsilon^2+\gamma^2(k))|\epsilon|}$$
	it follows that  
	
	\begin{equation}
	\begin{split}
	\left|\frac{\partial \ell}{\partial w_1 \partial w_2}(\mathbf{u})\right| &=   \frac{1}{\pi k} \lim_{\epsilon \to 0}{\left|\frac{\partial \theta_{\mathbf{u}_{\epsilon},\mathbf{w}^*}}{\partial w_2}\right|\left|\frac{\gamma(k)\cos\theta_{\mathbf{u}_{\epsilon},\mathbf{w}^*}}{\left\|\mathbf{u}_\epsilon\right\|} + 1\right|} \\ &\leq \frac{1}{\gamma(k)\pi k} \lim_{\epsilon \to 0}{\left|\frac{\gamma(k)\cos\theta_{\mathbf{u}_{\epsilon},\mathbf{w}^*}}{\left\|\mathbf{u}_\epsilon\right\|} + 1\right|} = 0
	\end{split}
	\end{equation}
\end{enumerate}
and thus $\frac{\partial \ell}{\partial w_1 \partial w_2}(\mathbf{u}) = 0$.

Taking derivatives of the gradient with respect to $w_1$ is easier because the expressions in \eqref{eq:loss_gradient} that depend on $\theta_{\mathbf{w},\mathbf{w}^*}$ and $\frac{\mathbf{w}}{\left\|\mathbf{w}\right\|}$ are constant. Therefore,
$$\frac{\partial \ell}{\partial w_1^2}(\mathbf{u}) = \frac{k + \beta(k)}{k^2}$$
and 
$$\frac{\partial \ell}{\partial w_2 \partial w_1}(\mathbf{u}) = 0$$

Finally let $\tilde{\mathbf{u}}_{\epsilon} = (0,-\gamma(k),\epsilon,0,...,0)$ then it is easy to see that $$\frac{\partial \ell}{\partial {w_2} \partial w_3}(\mathbf{u}) = \lim_{\epsilon \to 0}{\frac{\nabla \ell (\tilde{\mathbf{u}}_{\epsilon})_{w_2} - \nabla\ell(\mathbf{u})_{w_2}}{\epsilon}} = 0$$.

Therefore, overall we see that $\nabla^2\ell(\mathbf{u})$ is a diagonal matrix with zeros and $\frac{k + \beta(k)}{k^2} > 0$ on the diagonal, which proves our claim.

\subsection {Proof of Theorem 5.2}%\ref{non_overlapping_thm}} \label{section_B2}

For the following lemmas let $\mathbf{w}_{t+1} = \mathbf{w}_t - \lambda\nabla \ell(\mathbf{w}_t)$, $\theta_t$ be the angle between $\mathbf{w}_t$ and $\mathbf{w}^*$ ($t \geq 0$) 	and define $\tilde{\lambda}=\alpha(k)\lambda$ where $\alpha(k) = \frac{1}{k} + \frac{k^2-k}{\pi k^2}$. Note that $\alpha(k) \leq 1$ for all $k \geq 1$ The following lemma shows that for $\lambda < 1$, the angle between $\mathbf{w}_t$ and $\mathbf{w}^*$ decreases in each iteration. 

\begin{lem}
	\label{increasingTheta_conv}
	If $0 < \theta_t < \pi$ and $\lambda < 1$ then $\theta_{t+1} < \theta_t$.
\end{lem}

\begin{proof}
	This follows from the fact that adding $$-\frac{\lambda}{k^2}\Bigg(k + \frac{k^2-k}{\pi} -  \frac{k\left \| \mathbf{w}^* \right \|}{\pi\left \| \mathbf{w}_t \right \|} \sin \theta_t - \frac{k^2-k}{\pi}\frac{\left \| \mathbf{w}^* \right \|}{\left \| \mathbf{w}_t \right \|} \Bigg) \mathbf{w}_t$$ to $\mathbf{w}_t$ does not change $\theta_t$ for $\lambda<1$, since $\frac{k + \frac{k^2-k}{\pi}}{k^2} \leq 1$ for $k \geq 1$. In addition, adding $\frac{\lambda}{\pi k}\Big(\pi - \theta \Big) \mathbf{w}^*$ decreases $\theta_t$.
\end{proof}

We will need the following two lemmas to establish a lower bound on $\left\|\mathbf{w}_t\right\|$.
\begin{lem}
	\label{greaterThanPi/2_conv}
	If $\frac{\pi}{2} < \theta_t < \pi$ then $\left \| \mathbf{w}_{t+1} \right \| \geq \frac{\sin{\theta_{t}}}{\sin{\theta_{t+1}}} \min\{\left \| \mathbf{w}_t \right \|,\frac{\left \| \mathbf{w}^* \right \| \sin \theta_t}{\alpha(k)\pi}\} $.
\end{lem}

\begin{proof}
	Let $$ \mathbf{u}_t = \mathbf{w}_t -\frac{\lambda}{k^2}\Bigg(k + \frac{k^2-k}{\pi} -  \frac{k\left \| \mathbf{w}^* \right \|}{\pi\left \| \mathbf{w}_t \right \|} \sin \theta_t - \frac{k^2-k}{\pi}\frac{\left \| \mathbf{w}^* \right \|}{\left \| \mathbf{w}_t \right \|} \Bigg) \mathbf{w}_t$$
	
	Notice that if $\left \| \mathbf{w}_t \right \| \leq  \frac{\left \| \mathbf{w}^* \right \| \sin \theta_t}{\alpha(k)\pi}  $ then
	\begin{equation}
	\begin{split}
	\left \| \mathbf{u}_t \right \| &= (1-\tilde{\lambda})\left \| \mathbf{w}_t \right \| + \frac{\lambda \left \| \mathbf{w}^* \right \|}{\pi k} \sin \theta_t + \frac{\lambda(k^2-k)\left \| \mathbf{w}^* \right \|}{\pi k^2} \\ &\geq (1-\tilde{\lambda})\left \| \mathbf{w}_t \right \| + \frac{\lambda k\left \| \mathbf{w}^* \right \|\sin \theta_t}{\pi k^2}  + \frac{\lambda(k^2-k)\left \| \mathbf{w}^* \right \|\sin \theta_t}{\pi k^2}  \\ &=  (1-\tilde{\lambda})\left \| \mathbf{w}_t \right \| + \frac{\tilde{\lambda}\left \| \mathbf{w}^* \right \|\sin \theta_t}{\alpha(k)\pi}  \geq  \left \| \mathbf{w}_t \right \|
	\end{split}
	\end{equation}
	
	Similarly, if $\left \| \mathbf{w}_t \right \| \geq  \frac{\left \| \mathbf{w}^* \right \| \sin \theta_t}{\alpha(k)\pi} $ then $\left \| \mathbf{u}_t \right \| \geq \frac{\left \| \mathbf{w}^* \right \| \sin \theta_t}{\alpha(k)\pi} $. Furthermore, by a simple geometric observation we see that $\left \| \mathbf{w}_{t+1} \right \| \cos (\theta_{t+1} - \frac{\pi}{2}) = \left \| \mathbf{u}_t \right \| \cos (\theta_{t} - \frac{\pi}{2})$ if $\theta_{t+1} > \frac{\pi}{2}$ and $\left \| \mathbf{w}_{t+1} \right \| \cos (\frac{\pi}{2} - \theta_{t+1}) = \left \| \mathbf{u}_t \right \| \cos (\theta_{t} - \frac{\pi}{2})$ if $\theta_{t+1} \leq \frac{\pi}{2}$. This is equivalent to $\left \| \mathbf{w}_{t+1} \right \| = \frac{\sin{\theta_{t}}}{\sin{\theta_{t+1}}} \left \| \mathbf{u}_t \right \|$. It follows that $\left \| \mathbf{w}_{t+1} \right \| \geq \frac{\sin{\theta_{t}}}{\sin{\theta_{t+1}}} \min\{\left \| \mathbf{w}_t \right \|,\frac{\left \| \mathbf{w}^* \right \| \sin \theta_t}{\alpha(k)\pi}\}$ as desired.
\end{proof}

\begin{lem}
	\label{lessThanPi/2_conv}
	If $0 < \theta_t \leq \frac{\pi}{2}$ and $0 < \lambda < \frac{1}{2}$ then $\left \| \mathbf{w}_{t+1} \right \| \geq \min\{\left \| \mathbf{w}_t \right \|,\frac{\left \| \mathbf{w}^* \right \|}{8}\}$
\end{lem}

\begin{proof}
	First assume that $k\geq 2$. Let $\uu_t$ be as in Lemma \ref{greaterThanPi/2_conv}, then $$\left \| \mathbf{u}_t \right \| \geq (1-\tilde{\lambda})\left \| \mathbf{w}_t \right \| + \frac{\tilde{\lambda}(k^2-k)\left \| \mathbf{w}^* \right \|}{\alpha(k)\pi k^2}$$
	It follows that if $\left\|\mathbf{w}_t\right\| \geq \frac{(k^2-k)\left \| \mathbf{w}^* \right \|}{\alpha(k)\pi k^2} \geq \frac{\left\|\mathbf{w}^*\right\|}{2\pi}$ then $\left\|\mathbf{u}_t\right\| \geq \frac{\left\|\mathbf{w}^*\right\|}{2\pi}$. Otherwise if $\left\|\mathbf{w}_t\right\| \leq \frac{(k^2-k)\left \| \mathbf{w}^* \right \|}{\alpha(k)\pi k^2}$ then $\left\|\mathbf{u}_t\right\| \geq \left\|\mathbf{w}_t\right\|$. Since $\mathbf{w}_{t+1} = \mathbf{u}_t + \frac{\lambda }{\pi k}\Big(\pi - \theta \Big) \mathbf{w}^*$ and $0 < \theta_t \leq \frac{\pi}{2}$ we have $\left \| \mathbf{w}_{t+1} \right \| \geq \left \| \mathbf{u}_t \right \| \geq \min\{\frac{\left \| \mathbf{w}^* \right \|}{2\pi},\left\|\mathbf{w}_t\right\|\}$.
	
	Now let $k=1$. Note that in this case $\tilde{\lambda} = \lambda$. First assume that $\theta_t < \frac{\pi}{3}$. If $\left \| \mathbf{w}_t \right \| \geq \frac{\left \| \mathbf{w}^* \right \|}{4}$  then, using the same notation as in Lemma \ref{greaterThanPi/2_conv}, $\left \| \mathbf{u}_t \right \| \geq (1-\lambda)\left \| \mathbf{w}_t \right \| + \frac{\lambda \left \| \mathbf{w}^* \right \|\sin \theta_t}{\pi}  \geq \frac{\left \| \mathbf{w}_t \right \|}{2}  \geq \frac{\left \| \mathbf{w}^* \right \|}{8}$. Since $\mathbf{w}_{t+1} = \mathbf{u}_t + \frac{\lambda }{\pi }\Big(\pi - \theta_t \Big) \mathbf{w}^*$ and $0 < \theta_t \leq \frac{\pi}{2}$ we have $\left \| \mathbf{w}_{t+1} \right \| \geq \left \| \mathbf{u}_t \right \| \geq \frac{\left \| \mathbf{w}^* \right \|}{8}$. If $\left \| \mathbf{w}_t \right \| < \frac{\left \| \mathbf{w}^* \right \|}{4}$ then by the facts $0 < \theta_t \leq \frac{\pi}{2}$ and $\cos{\theta_t} > \frac{1}{2}$ we get 
	
	\begin{equation}
	\begin{split}
	\left \| \mathbf{w}_{t+1} \right \|^2 &= \left \| \mathbf{u}_t  \right \|^2 + 2\left \| \mathbf{u}_t \right \|\left \| \frac{\lambda }{\pi}\Big(\pi - \theta_t \Big) \mathbf{w}^* \right \| \cos{\theta_t} + \left \| \frac{\lambda }{\pi}\Big(\pi - \theta_t \Big) \mathbf{w}^* \right \|^2 \\ &\geq (1-\lambda)^2\left \| \mathbf{w}_t \right \|^2 + \frac{(1-\lambda)\lambda}{2}\left \| \mathbf{w}_t \right \|\left \| \mathbf{w}^* \right \| + \frac{\lambda^2}{4}\left \| \mathbf{w}^* \right \|^2 \\ &\geq (1-\lambda)^2\left \| \mathbf{w}_t \right \|^2 + 2(1-\lambda)\lambda\left \| \mathbf{w}_t \right \|^2 + 4\lambda^2\left \| \mathbf{w}_t \right \|^2 \\ &= (1+3\lambda^2)\left \| \mathbf{w}_t \right \|^2 \geq \left \| \mathbf{w}_t \right \|^2
	\end{split}
	\end{equation}
	
	Finally, assume $\theta_t \geq \frac{\pi}{3}$. As in the proof of Lemma \ref{greaterThanPi/2_conv}, if $\left \| \mathbf{w}_t \right \| \geq  \frac{\left \| \mathbf{w}^* \right \| \sin \theta_t}{\pi}   \geq \frac{\sqrt{3}}{2}\frac{\left \| \mathbf{w}^* \right \|}{\pi}$ then $\left \| \mathbf{w}_{t+1} \right \| \geq \left \| \mathbf{u}_t \right \| \geq \frac{\sqrt{3}}{2}\frac{\left \| \mathbf{w}^* \right \|}{\pi}$. Otherwise, if $\left \| \mathbf{w}_t \right \| <  \frac{\left \| \mathbf{w}^* \right \| \sin \theta_t}{\pi} $ then $\left \| \mathbf{w}_{t+1} \right \| \geq \left \| \mathbf{u}_t \right \| \geq \left \| \mathbf{w}_t \right \|$. This concludes our proof. 
\end{proof}

We can now show that in each iteration $\left\|\mathbf{w}_t\right\|$ is bounded away from $0$ by a constant.

\begin{prop}
	\label{normBoundProp_conv}
	Assume GD is initialized at $\mathbf{w}_0$ such that $\theta_0 \neq \pi$ and runs for $T$ iterations with learning rate $0 < \lambda < \frac{1}{2}$. Then for all $0 \leq t \leq T$, $$\left \| \mathbf{w}_t \right \| \geq \min\{\left \| \mathbf{w}_0 \right \| \sin{\theta_0},\frac{\left \| \mathbf{w}^* \right \| \sin^2 \theta_0}{\alpha(k)\pi}, \frac{\left \| \mathbf{w}^* \right \|}{8}\}$$  
\end{prop}

\begin{proof}
	Let $\theta_0 > \theta_1 > ... > \theta_T $ (by Lemma \ref{increasingTheta_conv}). Let $i$ be the last index such that $\theta_i > \frac{\pi}{2}$ (if such $i$ does not exist let $i=-1$). Since $\sin{\theta_{j}} > \sin{\theta_0}$ for all $0 \leq j \leq i$, by applying Lemma \ref{greaterThanPi/2_conv} at most $j+1$ times we have $$\left \| \mathbf{w}_{j+1} \right \| \geq \min\{\left \| \mathbf{w}_0 \right \| \sin{\theta_0},\frac{\left \| \mathbf{w}^* \right \| \sin^2 \theta_0}{\alpha(k)\pi}\}$$ for all $0 \leq j \leq i$.
	
	Finally, by Lemma \ref{lessThanPi/2_conv} and the fact that $\theta_j \leq \frac{\pi}{2}$ for all $i < j \leq T$, we get $$\left \| \mathbf{w}_{j} \right \| \geq \min\{\left \| \mathbf{w}_{i+1} \right \|, \frac{\left \| \mathbf{w}^* \right \|}{8} \}$$
	for all $i+1 < j \leq T$, from which the claim follows.   	
\end{proof}

The following lemma shows that $\nabla\ell$ is Lipschitz continuous at points that are bounded away from 0.

\begin{lem}
	\label{lipschitzLem_conv}
	Assume $\left \| \mathbf{w}_1 \right \|, \left \| \mathbf{w}_2 \right \| \geq M$, $\mathbf{w}_1,\mathbf{w}_2$ and $\mathbf{w}^*$ are on the same two dimensional half-plane defined by $\ww^*$, then $$\left \| \nabla \ell(\mathbf{w}_1) - \nabla \ell(\mathbf{w}_2) \right \| \leq L \left \| \mathbf{w}_1 - \mathbf{w}_2 \right \|$$ for $L= 1+ \frac{3\left\|\mathbf{w}^*\right\|}{M}$.
\end{lem}
\begin{proof}
	Recall that by equality \eqref{gradRelu}, $$\frac{\partial g}{\partial \mathbf{w}}(\mathbf{w},\mathbf{w}^*) = \frac{1}{2\pi}\left \| \mathbf{w}^* \right \| \frac{\mathbf{w}}{\left \| \mathbf{w} \right \|} \sin \theta_{\ww,\ww^*} + \frac{1}{2\pi}\Big(\pi - \theta_{\ww,\ww^*} \Big) \mathbf{w}^*$$

	Let $\theta_1$ and $\theta_2$ be the angles between $\ww_1$,$\ww^*$ and $\ww_2$,$\ww^*$, respectively. By the inequality $\frac{x_0\sin{x}}{\sin{x_0}} \geq x$ for $0 \leq x \leq x_0 < \pi$ and since $\frac{|\theta_1 - \theta_2 |}{2} \leq \frac{\pi}{2}$ we have
	$$\frac{|\theta_1 - \theta_2 |}{2} \leq \frac{\pi \sin\frac{|\theta_1 - \theta_2 |}{2}}{2}$$
	Furthermore $\left \| \mathbf{w}_1 - \mathbf{w}_2 \right \|$ is minimized (for fixed angles $\theta_1$ and $\theta_2$) when $\left \| \mathbf{w}_1 \right \| = \left \| \mathbf{w}_2 \right \| = M$ and is equal to $2M\sin\frac{|\theta_1 - \theta_2 |}{2}$. Thus, under our assumptions we have,
	$$\frac{|\theta_1 - \theta_2 |}{2} \leq \frac{\pi \sin\frac{|\theta_1 - \theta_2 |}{2}}{2} \leq \frac{\pi\left \| \mathbf{w}_1 - \mathbf{w}_2 \right \| }{4M}$$
	Thus we get $$\left \| \frac{1}{2\pi}\Big(\pi - \theta_1 \Big) \mathbf{w}^* - \frac{1}{2\pi}\Big(\pi - \theta_2 \Big) \mathbf{w}^* \right \| \leq \frac{\left \| \mathbf{w}^* \right \|}{4M}\left \| \mathbf{w}_1 - \mathbf{w}_2 \right \|$$
	
	For the first summand, we will first find the parameterization of a two dimensional vector of length $\sin\theta$ where $\theta$ is the angle between the vector and the positive $x$ axis. Denote this vector by $(a,b)$, then the following holds
	$$a^2 + b^2 = \sin^2\theta $$
	and
	$$\frac{b}{a} = \tan\theta$$
	The solution to these equations is $(a,b) = (\frac{\sin2\theta}{2},\sin^2\theta) $. Hence (here we use the fact that $\ww_1$,$\ww_2$ are on the same half-plane) 
	\begin{equation}
	\begin{split}
	\left \| \frac{1}{2\pi}\left \| \mathbf{w}^* \right \| \frac{\mathbf{w}_1}{\left \| \mathbf{w}_1 \right \|} \sin \theta_1 - \frac{1}{2\pi}\left \| \mathbf{w}^* \right \| \frac{\mathbf{w}_2}{\left \| \mathbf{w}_2 \right \|} \sin \theta_2 \right \| &= \frac{1}{2\pi}\left \| \mathbf{w}^* \right \| \sqrt{\Big(\frac{\sin2\theta_1}{2} - \frac{\sin2\theta_2}{2}\Big)^2 + \Big(\sin^2\theta_1 - \sin^2\theta_2\Big)^2} \\ &\leq \frac{1}{2\pi}\left \| \mathbf{w}^* \right \| \sqrt{(\theta_1 - \theta_2)^2 + 4(\theta_1 - \theta_2)^2} \\ &\leq \frac{\sqrt{5}}{\pi}\left \| \mathbf{w}^* \right \| \frac{\pi\left \| \mathbf{w}_1 - \mathbf{w}_2 \right \| }{4M} \\ &= \frac{\sqrt{5}\left \| \mathbf{w}^* \right \|}{4M}\left \| \mathbf{w}_1 - \mathbf{w}_2 \right \|
	\end{split}
	\end{equation}
	
	where the first inequality follows from the fact that $|\sin x - \sin y| \leq |x-y|$ and the second inequality from previous results. In conclusion, we have $$\left \| \frac{\partial g}{\partial \ww}(\mathbf{w}_1,\mathbf{w}^*) - \frac{\partial g}{\partial \ww}(\mathbf{w}_2,\mathbf{w}^*) \right \| \leq \frac{(\sqrt{5} + 1)\left \| \mathbf{w}^* \right \|}{4M}\left \| \mathbf{w}_1 - \mathbf{w}_2 \right \|$$
	Similarly, in order to show that the function $f(\ww) = \frac{\ww}{\left \| \ww \right \|}$ is Lipschitz continuous,  we parameterize the unit vector by $(\cos\theta,\sin\theta)$ where $\theta$ is the angle between the vector and the positive $x$ axis. We now obtain 
	\begin{equation}
	\begin{split}
	\left \| \frac{\mathbf{w}_1}{\left \| \mathbf{w}_1 \right \|} - \frac{\mathbf{w}_2}{\left \| \mathbf{w}_2 \right \|} \right \| &= \sqrt{(\cos \theta_1 - \cos \theta_2)^2 + (\sin \theta_1 - \sin \theta_2)^2} \\ &\leq \sqrt{2(\theta_1 - \theta_2)^2} \\ &\leq \frac{ \pi\left \| \mathbf{w}_1 - \mathbf{w}_2 \right \|}{\sqrt{2}M} 
	\end{split}
	\end{equation}
	
	Now we can conclude that 
	\begin{equation}
	\begin{split}
	\left \| \nabla \ell(\mathbf{w}_1) - \nabla \ell(\mathbf{w}_2) \right \| &\leq \big(\frac{1}{k}+ \frac{k^2-k}{\pi k^2}\big)\left \| \mathbf{w}_1 - \mathbf{w}_2 \right \| + \frac{2}{k}\left \| \frac{\partial g}{\partial \ww}(\mathbf{w}_1,\mathbf{w}^*) - \frac{\partial g}{\partial \ww}(\mathbf{w}_2,\mathbf{w}^*) \right \| \\ &+ \Big(\frac{(k^2-k)\left \| \mathbf{w}^* \right \|}{\pi k^2}\Big)\left \| \frac{\mathbf{w}_1}{\left \| \mathbf{w}_1 \right \|} - \frac{\mathbf{w}_2}{\left \| \mathbf{w}_2 \right \|} \right \| \\ &\leq \Big(\frac{1}{k}+ \frac{k^2-k}{\pi k^2} + \frac{(k^2-k)\left \| \mathbf{w}^* \right \|}{\sqrt{2}Mk^2}  +\frac{(\sqrt{5} + 1)\left \| \mathbf{w}^* \right \|}{2Mk}\Big)\left \| \mathbf{w}_1 - \mathbf{w}_2 \right \| \\ &\leq 1 + \frac{\left \| \mathbf{w}^* \right \|}{\sqrt{2}M}  +\frac{(\sqrt{5} + 1)\left \| \mathbf{w}^* \right \|}{2M} \\ &\leq 1+ \frac{3\left\|\mathbf{w}^*\right\|}{M}
	\end{split}
	\end{equation}
\end{proof}

Given that $\ell$ is Lipschitz continuous we can now follow standard optimization analysis (\citep{nesterov2004introductory}) to show that $\lim_{t \to \infty}{ \left \| \nabla \ell(\bs{w}_t) \right \|} = 0$.

\begin{prop}
	\label{prop:gradient_series_bound}
	Assume GD is initialized at $\bs{w}_0$ such that $\theta_0 \neq \pi$ and runs with a constant learning rate $0 < \lambda < \min\{\frac{2}{L},\frac{1}{2}\}$ where $L = \tilde{O}(1)$. Then for all $T$ $$\sum_{t=0}^{T}{\left \| \nabla \ell(\bs{w}_t) \right \|^2} \leq \frac{1}{\lambda(1-\frac{\lambda}{2}L)}\ell(\bs{w}_0)$$
\end{prop}
\begin{proof}
	We will need the following lemma
	\begin{lem}
		\label{nemirovski_lem_conv}
		Let $f:\mathbb{R}^n \to \mathbb{R}$ be a continuously differentiable function on a set $D \subseteq \mathbb{R}^n$ and $x,y \in D$ such that for all $0\leq\tau\leq 1$, $x + \tau(y-x) \in D$ and $\left \| \nabla f(x + \tau(y-x)) - \nabla f(x) \right \| \leq L \left \| x - y \right \|$. Then we have $$|f(y) - f(x) -\langle \nabla f(x),y-x \rangle| \leq \frac{L}{2} \left \| x-y \right \|^2$$
	\end{lem}
	
	\begin{proof}
		The proof exactly follows the proof of Lemma 1.2.3 in \citep{nesterov2004introductory} and note that the proof only requires Lipschitz continuity of the gradient on the set $S=\{x + \tau(y-x) \mid 0\leq\tau\leq 1\}$ and that $S \subseteq D$.
	\end{proof}
	
	By Proposition \ref{normBoundProp_conv}, for all $t$, $\left \| \mathbf{w}_t \right \| \geq M'$ where $$M'=\min\{\left \| \mathbf{w}_0 \right \| \sin{\theta_0},\frac{\left \| \mathbf{w}^* \right \| \sin^2 \theta_0}{\alpha(k)\pi}, \frac{\left \| \mathbf{w}^* \right \|}{8}\}$$. Furthermore, by a simple geometric observation we have $$\min_{0 \leq \tau \leq 1,\left\|\mathbf{w}_1\right\|,\left \| \mathbf{w}_2 \right \| \geq M', \arccos\Big(\frac{\mathbf{w}_1 \cdot \mathbf{w}_2}{\left \| \mathbf{w}_1 \right \| \left \| \mathbf{w}_2 \right \|}\Big) = \theta}{\left \| \tau \mathbf{w}_1 + (1-\tau)\mathbf{w}_2 \right \|} = M'\cos{\frac{\theta}{2}}$$.
	
	It follows by Lemma \ref{lipschitzLem_conv} that for any $t$ and $\xx_1,\xx_2 \in S_t \triangleq \{\mathbf{w}_t + \tau(\mathbf{w}_{t+1} - \mathbf{w}_t) \mid 0 \leq \tau \leq 1\}$, $$\left \| \nabla \ell(\xx_1) - \nabla \ell(\xx_2) \right \| \leq L\left \| \xx_1 - \xx_2 \right \|$$ where $L=1+ \frac{3\left\|\mathbf{w}^*\right\|}{M}$ and $M=M'\cos{\frac{\theta_0}{2}}$ (Note that $\cos{\frac{\theta_t-\theta_{t+1}}{2}} \geq \cos{\frac{\theta_0}{2}}$ for all $t$ by Lemma \ref{increasingTheta_conv}). 
	
	Hence by Lemma \ref{nemirovski_lem_conv}, for any $t$ we have
	\begin{equation}
	\begin{split}
	\ell(\mathbf{w}_{t+1}) &\leq \ell(\mathbf{w}_t) + \langle \nabla \ell(\mathbf{w}_t),\mathbf{w}_{t+1} - \mathbf{w}_t \rangle + \frac{L}{2} \left \| \mathbf{w}_{t+1} - \mathbf{w}_t \right \|^2 \\ &= \ell(\mathbf{w}_t) - \lambda(1-\frac{\lambda}{2}L)\left \| \nabla \ell(\mathbf{w}_t) \right \|^2
	\end{split}
	\end{equation}

	which implies that $$\sum_{t=0}^{T}{\left \| \nabla \ell(\mathbf{w}_t) \right \|^2} \leq \frac{1}{\lambda(1-\frac{\lambda}{2}L)}\Big(\ell(\mathbf{w}_0) - \ell(\mathbf{w}_T)\Big) \leq \frac{1}{\lambda(1-\frac{\lambda}{2}L)}\ell(\mathbf{w}_0)$$
\end{proof}

We are now ready to prove the theorem.
\medskip \par \noindent {\it Proof of Theorem \ref{non_overlap_poly_conv}. } First, we observe that for a randomly initialized point $\mathbf{w}_0$, $0 \leq \theta_0\leq \pi(1 - \delta)$ with probability $1-\delta$. Hence by Proposition \ref{prop:gradient_series_bound} we have for $L = 1+ \frac{3\left\|\mathbf{w}^*\right\|}{M}$ where $M=\min\{ \sin(\pi(1 - \delta)),\frac{ \sin ^2 (\pi(1 - \delta))}{\alpha(k)\pi} , \frac{1}{8}\}\cos(\frac{\pi(1 - \delta)}{2})$ and $\alpha(k)=k+\frac{k^2-k}{\pi}$, and for $\lambda = \frac{1}{L}$ (we assume w.l.o.g. that $L > 2$),

$$\sum_{t=0}^{T}{\left \| \nabla \ell(\mathbf{w}_t) \right \|^2} \leq \frac{1}{\lambda(1-\frac{\lambda}{2}L)}\ell(\mathbf{w}_0) = 2L\ell(\mathbf{w}_0) \leq \frac{4L}{k^2}\big(\frac{k}{2} + \frac{k^2-k}{2\pi}\big)$$

Therefore,

$$\min_{0 \leq t \leq T}\{{\left\|\nabla \ell(\mathbf{w}_t)\right\|}^2\} \leq \frac{\frac{4L}{k^2}\big(\frac{k}{2} + \frac{k^2-k}{2\pi}\big)}{T}$$

It follows that gradient descent reaches a point $\mathbf{w}_t$ such that $\left \| \nabla \ell(\mathbf{w}_t) \right\| < \epsilon$ after $T$ iterations where $$T > \frac{\Big(\frac{4L}{k^2}\big(\frac{k}{2} + \frac{k^2-k}{2\pi}\big)\Big)^2}{\epsilon^2} $$

We will now show that if $\left \| \nabla \ell(\mathbf{w}_t) \right\| < \epsilon$ then $\mathbf{w}_t$ is $O(\sqrt{\epsilon})$-close to the global minimum $\mathbf{w}^*$. First note that if $\frac{\pi}{2} \leq \theta_t \leq \pi(1 - \delta)$ then a vector of the form $\bs{v} = \alpha\bs{w}^*+\beta\bs{w}$ where $\alpha \geq 0$ is of minimal norm equal to $\alpha\sin(\pi-\theta_t)\left\|\bs{w}^*\right\|$ when it is perpendicular to $\bs{w}$. Since the gradient is a vector of this form, we have $\left\|\nabla \ell(\mathbf{w}_t)\right\| > \frac{\pi\delta\left\|\mathbf{w}^*\right\|\sin\pi\delta}{\pi k} \geq \frac{\delta\sin\pi\delta}{k}\geq \epsilon$. Hence, from now on we assume that $0 \leq \theta_t < \frac{\pi}{2}$.

Similarly to the previous argument, we have $$\epsilon > \left\|\nabla \ell(\mathbf{w}_t)\right\| > \frac{\left\|\mathbf{w}^*\right\|(\pi - \frac{\pi }{2})\sin\theta_t}{\pi k} \geq \frac{\sin\theta_t}{2k}$$

Hence, $\theta_t < \arcsin(2k\epsilon) = O(\epsilon)$. It follows by the triangle inequality that 
\begin{equation}
\begin{split}
k^2\epsilon > k^2\left\|\nabla \ell(\mathbf{w}_t)\right\| &= \left\|\Bigg(k + \frac{k^2-k}{\pi} -  \frac{k\left \| \mathbf{w}^* \right \|}{\pi\left \| \mathbf{w}_t \right \|} \sin \theta_t - \frac{k^2-k}{\pi}\frac{\left \| \mathbf{w}^* \right \|}{\left \| \mathbf{w}_t \right \|} \Bigg) \mathbf{w}_t - \frac{k(\pi - \theta_t )}{\pi} \mathbf{w}^* \right\| \\ &\geq \left\|\Bigg(k + \frac{k^2-k}{\pi} -\frac{k^2-k}{\pi}\frac{\left \| \mathbf{w}^* \right \|}{\left \| \mathbf{w}_t \right \|} \Bigg) \mathbf{w}_t - k\mathbf{w}^*\right\| - \frac{k\left \| \mathbf{w}^* \right \|}{\pi} \sin \theta_t - \frac{k\theta_t\left\|\mathbf{w}^*\right\|}{\pi} \\ &\geq \left\|\Bigg(k + \frac{k^2-k}{\pi} -\frac{k^2-k}{\pi}\frac{\left \| \mathbf{w}^* \right \|}{\left \| \mathbf{w}_t \right \|} \Bigg)\mathbf{w}_t - \frac{k\left\|\mathbf{w}^*\right\|}{\left\|\mathbf{w}_t\right\|}\mathbf{w}_t \right\| \\ &- \left\|k\mathbf{w}^*- \frac{k\left\|\mathbf{w}^*\right\|}{\left\|\mathbf{w}_t\right\|}\mathbf{w}_t\right\| - \frac{k\left \| \mathbf{w}^* \right \|}{\pi} \sin \theta_t - \frac{k\theta_t\left\|\mathbf{w}^*\right\|}{\pi} \\ &\geq \big(k + \frac{k^2-k}{\pi}\big)|\left\|\mathbf{w}_t\right\| - \left\|\mathbf{w}^*\right\|| - k\left\|\mathbf{w}^*\right\|\theta_t - \frac{k\left \| \mathbf{w}^* \right \|}{\pi} \sin \theta_t - \frac{k\theta_t\left\|\mathbf{w}^*\right\|}{\pi}
\end{split}
\end{equation}

where the last inequality follows since the arc of a circle is larger than its corresponding segment. 

Therefore we get $|\left\|\mathbf{w}_t\right\| - \left\|\mathbf{w}^*\right\|| < O(\epsilon)$. By the bounds on $\theta_t$ and $|\left\|\mathbf{w}_t\right\| - \left\|\mathbf{w}^*\right\||$ and the inequality $\cos x \geq 1-x$ for $x\geq0$, we can give an upper bound on $\left\|\bs{w}_t-\bs{w}^*\right\|$:

\begin{equation}
\begin{split}
{\left\| \mathbf{w}_t - \mathbf{w}^*\right\|}^2 &= {\left\|\mathbf{w}_t\right\|}^2 - 2\left\|\mathbf{w}_t\right\|\left\|\mathbf{w}^*\right\|\cos\theta_t + {\left\|\mathbf{w}^*\right\|}^2 \\ &= \left\|\mathbf{w}_t\right\|(\left\|\mathbf{w}_t\right\| - \left\|\mathbf{w}^*\right\|\cos\theta_t) + \left\|\mathbf{w}^*\right\|(\left\|\mathbf{w}^*\right\|-\left\|\mathbf{w}_t\right\|\cos\theta_t) \\ &\leq (\left\|\mathbf{w}^*\right\| + O(\epsilon))(O(\epsilon) + \theta_t\left\|\mathbf{w}^*\right\|) + \left\|\mathbf{w}^*\right\|(O(\epsilon^2) +  \theta_t\left\|\mathbf{w}^*\right\|) = O(\epsilon)
\end{split}
\end{equation}

Finally, to prove the claim it suffices to show that $\ell(\bs{w}) \leq d{\left\|\bs{w}-\bs{w}^*\right\|}^2$. Denote the input vector $\mathbf{x} = (\mathbf{x}_1,\mathbf{x}_2,...,\mathbf{x}_{k})$ where $\mathbf{x}_i \in \mathbb{R}^m$ for all $1 \leq i \leq k$. Then we get

\begin{equation}
\begin{split}
\ell(\bs{w}) &= \mathbb{E}_\mathbf{x}\Big[\frac{\sum_{i=1}^{k}{\sigma(\bs{w}^T\mathbf{x}_i)}}{k} - \frac{\sum_{i=1}^{k}{\sigma({\bs{w}^*}^T\mathbf{x}_i)}}{k}\Big]^2 \\ & \leq \mathbb{E}_\mathbf{x}\Big[\frac{\sum_{i=1}^{k}{|\sigma(\bs{w}^T\mathbf{x}_i)} - \sigma({\bs{w}^*}^T\mathbf{x}_i)|}{k}\Big]^2 \\ &
\leq \mathbb{E}_\mathbf{x}\Big[\frac{\sum_{i=1}^{k}{|\bs{w}^T\mathbf{x}_i} - {\bs{w}^*}^T\mathbf{x}_i|}{k}\Big]^2 \\ &
\leq \mathbb{E}_\mathbf{x}\Big[\frac{\sum_{i=1}^{k}{\left\|\bs{w} - \bs{w}^* \right\| \left\| \mathbf{x}_i \right\|} }{k}\Big]^2 \\ &
\leq {\left\|\bs{w} - \bs{w}^* \right\|}^2\mathbb{E}_\mathbf{x}{\left\| \mathbf{x} \right\|}^2 \\ &
= d{\left\|\bs{w} - \bs{w}^* \right\|}^2 
\end{split}
\end{equation}

where the second inequality follows from Lipschitz continuity of $\sigma$, the third inequality from the Cauchy-Schwarz inequality and the last equality since ${\left\| \mathbf{x} \right\|}^2$ follows a chi-squared distribution with $d$ degrees of freedom.

	\iffalse
	Since $$\left \| \mathbf{w}_{t+1} \right \| \leq  (1-\tilde{\lambda})\left \| \mathbf{w}_t \right \| + \frac{\tilde{\lambda}(n^2+\pi n)\left \| C \right \|}{\pi\alpha(n)n^2} $$ then if $M = \frac{(n^2+\pi n)\left \| C \right \|}{\pi\alpha(n)n^2}$ the inequality $\left \| \mathbf{w}_t \right \| \leq \max\{\left \| \mathbf{w}_0 \right \|,M\} $ implies $\left \| \mathbf{w}_{t+1} \right \| \leq \max\{\left \| \mathbf{w}_0 \right \|,M\} $. Thus by induction the sequence $\{\mathbf{w}_t\}_t$ is bounded. 
	
	Consider a converging subsequence $\{\mathbf{w}_{t_i}\}_i$ of $\{\mathbf{w}_t\}_t$ that converges to $\mathbf{w}^*$. By Proposition \ref{normBoundProp_conv}, for all $t$, $\mathbf{w}_t$ is bounded away from $0$ by a constant and thus by Lemma \ref{LossPropertiesLem_conv}, $\nabla l$ is continuous in the neighbourhood of the sequence and by the fact that $\lim_{i \to \infty}{ \left \| \nabla \ell(\mathbf{w}_{t_i}) \right \|} = 0$ we can deduce that $\nabla \ell(\mathbf{w}^*) = 0$. By Lemma $\ref{increasingTheta_conv}$, $\mathbf{w}_t$ is bounded away from $-(\frac{n^2-n}{n^2 + (\pi -1) n})C$ by a constant, thus by Lemma \ref{LossPropertiesLem_conv} we have $\mathbf{w}^* = C$. We have shown that any converging subsequence of $\{\mathbf{w}_t\}_t$ converges to $C$. This implies that $\lim_{t \to \infty}{ \mathbf{w}_t} = C$.
	\fi 
	{\hfill $\square$ \bigskip \par}

\section{Missing Proofs for Section \ref{overlapping_section}}
\label{supp:overlapping_gaussian}

\subsection{Proof of Proposition \ref{fourth_quad_prop}}
\label{section_C2}

Define $\ww_p = (w_2,w_1)$, $\ww^*_{p_1} = (0,-w^*)$ and $\ww^*_{p_2}=(w^*,0)$. We first prove the following lemma.

\begin{lem}
	\label{overlapping_gradient_lem}
	Let $l$ be defined as in \eqref{overlapping_loss}. Then
	\begin{align*}
		\nabla l(\mathbf{w}) &= \frac{1}{k^2}\Bigg[\big(k + \frac{k^2-3k+2}{\pi}\big)\mathbf{w} + \frac{2(k-1)\sin\theta_{\mathbf{w}_r,\mathbf{w}_l}}{\pi}\mathbf{w} \\ &+ \frac{(k-1)(\pi-\theta_{\mathbf{w}_r,\mathbf{w}_l})}{\pi}\mathbf{w}_p - 
		\frac{(k^2-3k+2)\left \| \mathbf{w}^* \right \|}{\pi \left\| \mathbf{w} \right\|}\mathbf{w} \\& - \frac{k\left \| \mathbf{w}^* \right \| \sin\theta_{\mathbf{w},\mathbf{w}^*}}{\pi \left\|\mathbf{w}\right\|}\mathbf{w} - \frac{k(\pi-\theta_{\mathbf{w},\mathbf{w}^*})}{\pi}\mathbf{w}^* \\
		&- \frac{(k-1)\sin\theta_{\mathbf{w}_l,\mathbf{w}^*_r}\left\|\mathbf{w}^*\right\|}{\pi\left\|\mathbf{w}\right\|}\mathbf{w} - \frac{(k-1)(\pi - \theta_{\mathbf{w}_l,\mathbf{w}^*_r})}{\pi}\mathbf{w}^*_{p_2} \\
		&- \frac{(k-1)\sin\theta_{\mathbf{w}_r,\mathbf{w}^*_l}\left\|\mathbf{w}^*\right\|}{\pi\left\|\mathbf{w}\right\|}\mathbf{w} - \frac{(k-1)(\pi - \theta_{\mathbf{w}_r,\mathbf{w}^*_l})}{\pi}\mathbf{w}^*_{p_1} \Bigg]
	\end{align*}
\end{lem}
\begin{proof}
The gradient does not follow immediately from Lemma \ref{kernel_gradient} because the loss has expressions with of the function $g$ but with different dependencies on the parameters in $\ww$. We will only calculate $\frac{\partial g(\mathbf{w}_r,\mathbf{w}_l)}{\partial \mathbf{w}}$, the other expressions are calculated in the same manner.

Recall that 
$$g(\mathbf{w}_r,\mathbf{w}_l) = \frac{1}{2\pi}{\left\|\mathbf{w}\right\|}^2(\sin\theta_{\mathbf{w}_r,\mathbf{w}_l}+(\pi - \theta_{\mathbf{w}_r,\mathbf{w}_l})\cos\theta_{\mathbf{w}_r,\mathbf{w}_l})$$
It follows that

\begin{equation}
	\frac{\partial g(\mathbf{w}_r,\mathbf{w}_l)}{\partial \mathbf{w}} = \frac{1}{\pi}(\sin\theta_{\mathbf{w}_r,\mathbf{w}_l}+(\pi - \theta_{\mathbf{w}_r,\mathbf{w}_l})\cos\theta_{\mathbf{w}_r,\mathbf{w}_l})\mathbf{w} 
	+\frac{1}{2\pi}{\left\|\mathbf{w}\right\|}^2(\pi - \theta_{\mathbf{w}_r,\mathbf{w}_l}) \frac{\partial \cos\theta_{\mathbf{w}_r,\mathbf{w}_l}}{\partial \mathbf{w}}
\end{equation}

Let $\mathbf{w}=(w_1,w_2)$ then $\cos\theta_{\mathbf{w}_r,\mathbf{w}_l} = \frac{w_1w_2}{w_1^2+w_2^2}$. Then,

$$\frac{\partial \cos\theta_{\mathbf{w}_r,\mathbf{w}_l}}{\partial w_1} = \frac{w_2(w_1^2+w_2^2)-2w_1^2w_2}{(w_1^2+w_2^2)^2}=\frac{w_2}{{\left\|\mathbf{w}\right\|}^2} - \frac{2w_1\cos\theta_{\mathbf{w}_r,\mathbf{w}_l}}{{\left\|\mathbf{w}\right\|}^2}$$
and
$$\frac{\partial \cos\theta_{\mathbf{w}_r,\mathbf{w}_l}}{\partial w_2} = \frac{w_1(w_1^2+w_2^2)-2w_2^2w_1}{(w_1^2+w_2^2)^2}=\frac{w_1}{{\left\|\mathbf{w}\right\|}^2} - \frac{2w_2\cos\theta_{\mathbf{w}_r,\mathbf{w}_l}}{{\left\|\mathbf{w}\right\|}^2}$$
or equivalently $\frac{\partial \cos\theta_{\mathbf{w}_r,\mathbf{w}_l}}{\partial \mathbf{w}}=\frac{\mathbf{w}_p}{{\left\|\mathbf{w}\right\|}^2}- \frac{2\mathbf{w}\cos\theta_{\mathbf{w}_r,\mathbf{w}_l}}{{\left\|\mathbf{w}\right\|}^2}$. It follows that
$$\frac{\partial g(\mathbf{w}_r,\mathbf{w}_l)}{\partial \mathbf{w}} = \frac{\sin\theta_{\mathbf{w}_r,\mathbf{w}_l}\mathbf{w}}{\pi} + \frac{(\pi - \theta_{\mathbf{w}_l,\mathbf{w}_r})}{2\pi}\mathbf{w}_p$$
\end{proof}
	We will prove that $\mathbf{w}_{t+1} \neq 0$ and that it is in the interior of the fourth quadrant.
	Denote $\mathbf{w} = \mathbf{w}_t$ and $\nabla l(\mathbf{w}) =\frac{1}{k^2}\big( B_1(\mathbf{w}) + B_2(\mathbf{w}) + B_3(\mathbf{w})\big)$ where 
	
	\begin{equation}
	\begin{split}
	B_1(\mathbf{w}) &= \big(k + \frac{k^2-3k+2}{\pi}\big)\mathbf{w} + \frac{2(k-1)\sin\theta_{\mathbf{w}_r,\mathbf{w}_l}}{\pi}\mathbf{w} -\frac{(k^2-3k+2)\left \| \mathbf{w}^* \right \|}{\pi \left\| A \right\|}\mathbf{w}\\&- \frac{k\left \| \mathbf{w}^* \right \| \sin\theta_{\mathbf{w},\mathbf{w}^*}}{\pi \left\|\mathbf{w}\right\|}\mathbf{w} -  \frac{(k-1)\sin\theta_{\mathbf{w}_l,\mathbf{w}^*_r}\left\|\mathbf{w}^*\right\|}{\pi\left\|\mathbf{w}\right\|}\mathbf{w} - \frac{(k-1)\sin\theta_{\mathbf{w}_r,\mathbf{w}^*_l}\left\|\mathbf{w}^*\right\|}{\pi\left\|\mathbf{w}\right\|}\mathbf{w} \\
	B_2(\mathbf{w}) &= \frac{(k-1)(\pi-\theta_{\mathbf{w}_r,\mathbf{w}_l})}{\pi}\mathbf{w}_p 
	\end{split}
	\end{equation}
	and
	\begin{equation}
	\begin{split}
	B_3(\mathbf{w}) &= - \frac{k(\pi-\theta_{\mathbf{w},\mathbf{w}^*})}{\pi}\mathbf{w}^* - \frac{(k-1)(\pi - \theta_{\mathbf{w}_l,\mathbf{w}^*_r})}{\pi}\mathbf{w}^*_{p_2} - \frac{(k-1)(\pi - \theta_{\mathbf{w}_r,\mathbf{w}^*_l})}{\pi}\mathbf{w}^*_{p_1}
	\end{split}
	\end{equation}
	
	Let $\mathbf{w}=(w,-mw)$ for $w,m\geq0$. Straightforward calculation shows that $\cos\theta_{\mathbf{w}_l,\mathbf{w}^*_r} = \frac{1}{\sqrt{2(1+m^2})}$ and $\cos\theta_{\mathbf{w}_r,\mathbf{w}^*_l} = \frac{m}{\sqrt{2(m^2+1)}}$. Hence $\frac{\pi}{4} \leq \theta_{\mathbf{w}_l,\mathbf{w}^*_r},\theta_{\mathbf{w}_r,\mathbf{w}^*_l} \leq \frac{\pi}{2}$. Since $\mathbf{w}$ is in the fourth quadrant we also have $\frac{3\pi}{4} \leq \theta_{\mathbf{w},\mathbf{w}^*} \leq \pi$. Therefore, adding $-\lambda B_3(\mathbf{w})$ can only increase $\left\|\mathbf{w}\right\|$. This follows since in the worst case (the least possible increase of $\left\|\mathbf{w}\right\|$)$$-B_3(\mathbf{w}) = \frac{k}{4}\mathbf{w}^* + \frac{k-1}{2}\mathbf{w}^*_{p_2} + \frac{k-1}{2}\mathbf{w}^*_{p_1} = (\frac{k-2}{4}w^*,-\frac{k-2}{4}w^*)$$
	which is in the fourth quadrant for $k\geq 2$.
	In addition, since $-\mathbf{w}_p$ is in the fourth quadrant then adding $-\lambda B_2(\mathbf{w})$ increases $\left\|\mathbf{w}\right\|$.
	
	If $\left\|\mathbf{w}\right\| < \frac{\left\|\mathbf{w}^*\right\|}{16}$ then $-B_1(\mathbf{w})$ points in the direction of $\mathbf{w}$ since in this case $-B_1(\mathbf{w}) = \alpha \mathbf{w}$ where $$\alpha \geq \Big(\frac{k^2-3k+2}{\pi} + \frac{(k-1)}{\pi} - \frac{k-1}{8\pi} - \frac{k^2-3k+2}{16\pi} - \frac{k}{16}\Big) \left\|\mathbf{w}^*\right\| > 0$$
	for $k \geq 2$. If $-B_1(\mathbf{w})$ points in the direction of $-\mathbf{w}$ then by the assumption that $\lambda \in (0,\frac{1}{3})$ we have $\left\|\lambda B_1(\mathbf{w})\right\| < \left\|\mathbf{w}\right\|$. Thus we can conclude that $\mathbf{w}_{t+1} \neq 0$.
	
	Now, let $\mathbf{w} = (w_1,w_2)$, $\theta_t$ be the angle between $\mathbf{w}=\mathbf{w}_t$ and the positive $x$ axis and first assume that $w_1 > -w_2$. In this case $-B_3(\mathbf{w})$ least increases (or even most decreases) $\theta_t$ when $$-B_3(\mathbf{w}) = \frac{k}{4}\mathbf{w}^* + \frac{3(k-1)}{4}\mathbf{w}^*_{p_2} + \frac{k-1}{2}\mathbf{w}^*_{p_1} =\Big(\frac{2k-3}{4}w^*,\frac{2-k}{4}w^*\Big)$$
	which is a vector in the fourth quadrant for $k \geq 2$. Otherwise, $-B_3(\mathbf{w})$ is a vector in the fourth quadrant as well. Note that we used the facts $\frac{\pi}{4} \leq \theta_{\mathbf{w}_l,\mathbf{w}^*_r},\theta_{\mathbf{w}_r,\mathbf{w}^*_l} \leq \frac{\pi}{2}$ and 
	$\frac{3\pi}{4} \leq \theta_{\mathbf{w},\mathbf{w}^*} \leq \pi$. Since $-\lambda B_1(\mathbf{w})$ does not change $\theta_t$ and $-\lambda B_2(\mathbf{w})$ increases $\theta_t$ but never to an angle greater than or equal to $\frac{\pi}{2}$, it follows that $0 < \theta_{t+1} < \frac{\pi}{2}$.
	
	If $w_1 \leq -w_2$ then by defining all angles with respect to the negative $y$ axis, we get the same argument as before. This shows that $\mathbf{w}_{t+1}$ is in the interior of the fourth quadrant, which concludes our proof.

\subsection{Proof of Proposition \ref{tight_bound_prop}}
\label{section_C3}

We will need the following auxiliary lemmas.

\begin{lem}
	\label{al_ar_lem}
	Let $\mathbf{w}$ be in the fourth quadrant, then $g(\mathbf{w}_l,\mathbf{w}_r) \geq \frac{1}{2\pi}\big(\frac{\sqrt{3}}{2} - \frac{\pi}{6}\big){\left\|\mathbf{w}\right\|}^2$.
\end{lem}
\begin{proof}
	First note that the function $s(\theta) = \sin\theta + (\pi -\theta)\cos\theta$ is decreasing as a function of $\theta \in [0,\pi]$. Let $\mathbf{w}=(w,-mw)$ for $w,m\geq0$. Straightforward calculation shows that $\cos\theta_{\mathbf{w}_l,\mathbf{w}_r} = -\frac{m}{m^2+1}$. As a function of $m \in [0,\infty)$, $\cos\theta_{\mathbf{w}_l,\mathbf{w}_r}$ is minimized for $m=1$ with value $-\frac{1}{2}$, i.e., when $\theta(\mathbf{w}_l,\mathbf{w}_r) = \frac{2\pi}{3}$ and this is the largest angle possible. Thus $g(\mathbf{w}_l,\mathbf{w}_r) \geq \frac{1}{2\pi}s(\frac{2\pi}{3})\big){\left\|\mathbf{w}\right\|}^2 = \frac{1}{2\pi}\big(\frac{\sqrt{3}}{2} - \frac{\pi}{6}\big){\left\|\mathbf{w}\right\|}^2$.
\end{proof}

\begin{lem}
	\label{trigo_lem}
	Let 
	\begin{equation}
	\begin{split}
	f(\theta) &= 2k\big(\sin(\frac{3\pi}{4} + \theta) + (\frac{\pi}{4} -\theta)\cos(\frac{3\pi}{4} + \theta)\big) \\ &+ \big(2k-2\big)\big(\sqrt{1-\frac{{\cos\theta}^2}{2}} + (\pi - \arccos{\frac{\cos\theta}{\sqrt{2}}})\frac{\cos\theta}{\sqrt{2}} \big) + \big(2k-2\big)\big(\sqrt{1-\frac{{\sin\theta}^2}{2}} + (\pi - \arccos{\frac{\sin\theta}{\sqrt{2}}})\frac{\sin\theta}{\sqrt{2}} \big)
	\end{split}
	\end{equation}, then in the interval $\theta \in [0,\frac{\pi}{4}]$, $f(\theta)$  is maximized at $\theta = \frac{\pi}{4}$ for all $k \geq 2$.  
\end{lem}

\begin{proof}
	We will maximize the function $\frac{f(\theta)}{2(k-1)} = \frac{k}{k-1}f_1(\theta) + f_2(\theta) + f_3(\theta)$ where $f_1(\theta),f_2(\theta),f_3(\theta)$ correspond to the three summands in the expression of $f(\theta)$. 
	
	%It suffices to prove the claim for $n=2$ since $f_1(\frac{\pi}{4}) = 0$, $f_2(\frac{\pi}{4}) > 0$, $f_3(\frac{\pi}{4}) > 0$ and $\frac{n}{k-1}$ is maximized for $n=2$. In other words, if $f_1(\theta)$ vanishes when $f(\theta)$ is maximized for $n=2$, then it will vanish when maximizing $f(\theta)$ for $n>2$, and $f_1(\theta)$ vanishes only for $\theta = \frac{\pi}{4}$.
	
	%Since $s(\theta) = \sin\theta + (\pi -\theta)\cos\theta$ is decreasing as a function of $\theta \in [0,\frac{\pi}{4}]$, $f_1(\theta)$ is maximized for $\theta = 0$ with value $f_1(0) = \frac{1}{\sqrt{2}} + \frac{\pi}{4\sqrt{(2)}} \leq \frac{2}{5}$, $f_2(\theta)$ is minimized for $\theta = \frac{\pi}{4}$ with value $f_2(\frac{\pi}{4}) = \frac{\sqrt{3}}{2} + \frac{\pi}{3} > 1$ and $f_3(\theta)$ is minimized for $\theta = 0$ with value $f_3(0) = 1$. It follows that the sum of minimal values of $f_2(\theta)$ and $f_3(\theta)$ are more than twice as large as the maximal value of $f_1(\theta)$. Since $\frac{2n}{2(k-1)} \leq 2$ for $n \geq 2$, in order to maximize ...
	
	Since for $h(x) = \sqrt{1-x^2} + (\pi - \arccos(x))x$ we have $h'(x) = \pi - \arccos(x)$, it follows that $f_2'(\theta) = -(\pi - \arccos{\frac{\cos\theta}{\sqrt{2}}})\frac{\sin\theta}{\sqrt{2}}$, $f_3'(\theta) = (\pi - \arccos{\frac{\sin\theta}{\sqrt{2}}})\frac{\cos\theta}{\sqrt{2}}$ and $f_1'(\theta) = -(\frac{\pi}{4} - \theta)\sin(\frac{3\pi}{4} + \theta)$. It therefore suffices to show that $$d_1(\theta) := (\pi - \arccos{\frac{\sin\theta}{\sqrt{2}}})\frac{\cos\theta}{\sqrt{2}} -(\pi - \arccos{\frac{\cos\theta}{\sqrt{2}}})\frac{\sin\theta}{\sqrt{2}} - \frac{k}{k-1}(\frac{\pi}{4} - \theta)\sin(\frac{3\pi}{4} + \theta) \geq 0$$
	for $\theta \in [0,\frac{\pi}{4}]$.
	
	By applying the inequalities $\arccos(x) \leq \frac{\pi}{2} - x$ for $x \in [0,1]$ and $\arccos(x) \geq \frac{\pi}{2} - x - \frac{1}{10}$ for $x \in [\frac{1}{2}, \frac{1}{\sqrt{2}}]$ we get $d_1(\theta) \geq d_2(\theta)$ where 
	\begin{equation}
	\begin{split}
	d_2(\theta) &= \big(\frac{\pi}{2} + \frac{\sin\theta}{\sqrt{2}}\big)\frac{\cos\theta}{\sqrt{2}} - \big(\frac{\pi}{2} + \frac{\cos\theta}{\sqrt{2}} + \frac{1}{10}\big)\frac{\sin\theta}{\sqrt{2}} - \frac{k}{k-1}(\frac{\pi}{4} - \theta)\sin(\frac{3\pi}{4} + \theta) \\ &= \frac{\pi}{2\sqrt{2}}\cos\theta - \big(\frac{\pi}{2\sqrt{2}} + \frac{1}{10\sqrt{2}}\big)\sin\theta - \frac{k}{k-1}(\frac{\pi}{4} - \theta)\sin(\frac{3\pi}{4} + \theta)
	\end{split}
	\end{equation}
	
	We notice that $d_2(0) \geq 0$ and $d_2(\frac{3}{4}) \geq 0$ for all $k \geq 2$. In addition, $$d_2'(\theta) = -\frac{\pi}{2\sqrt{2}}\sin\theta - \big(\frac{\pi}{2\sqrt{2}} + \frac{1}{10\sqrt{2}}\big)\cos\theta + \frac{k}{k-1}\sin(\frac{3\pi}{4}+\theta) -\frac{k}{k-1}(\frac{\pi}{4} - \theta)\cos(\frac{3\pi}{4} + \theta)$$
	and $d_2'(0) > 0$ for all $k \geq 2$. It follows that in order to show that $d_2(\theta) \geq 0 $ for $\theta \in [0,\frac{3}{4}]$ and $k \geq 2$, it suffices to show that $d_2''(\theta) \leq 0$ for $\theta \in [0,\frac{3}{4}]$ and $k \geq 2$. Indeed, 
	\begin{equation}
	\begin{split}
	d_2''(\theta) &= -\frac{\pi}{2\sqrt{2}}\cos\theta + \big(\frac{\pi}{2\sqrt{2}} + \frac{1}{10\sqrt{2}}\big)\sin\theta + \frac{2k}{k-1}\cos(\frac{3\pi}{4}+\theta) +\frac{k}{k-1}(\frac{\pi}{4} - \theta)\sin(\frac{3\pi}{4} + \theta) \\ &\leq \big(\frac{1}{10\sqrt{2}} + \frac{k}{k-1}\frac{\pi}{4}\big) \max\{\sin\theta,\sin(\frac{3\pi}{4} + \theta)\} + \frac{2k}{k-1}\cos(\frac{3\pi}{4}+\theta) \leq 0
	\end{split}
	\end{equation}

	for all $\theta \in [0,\frac{3}{4}]$ and $k \geq 2$. Note that the first inequality follows since $\cos\theta \geq \sin\theta$ and the second since $\cos(\frac{3\pi}{4}+\theta) \geq \max\{\sin\theta,\sin(\frac{3\pi}{4} + \theta)\}$, both for $\theta \in [0,\frac{3}{4}]$. This shows that $d_1(\theta) \geq 0$ for $\theta \in [0,\frac{3}{4}]$. 
	
	Now assume that $\theta \in [\frac{3}{4},\frac{\pi}{4}]$. Since $d_1(\frac{3}{4}) \geq 0$ and $d_1(\frac{\pi}{4}) \geq 0$, it suffices to prove that $d_1'(\theta) \leq 0$ for $\theta \in [\frac{3}{4},\frac{\pi}{4}]$. Indeed, for all $\theta \in [\frac{3}{4},\frac{\pi}{4}]$
	
	\begin{equation}
	\begin{split}
	d_1'(\theta) &= -(\pi - \arccos{\frac{\cos\theta}{\sqrt{2}}})\frac{\cos\theta}{\sqrt{2}} -(\pi - \arccos{\frac{\sin\theta}{\sqrt{2}}})\frac{\sin\theta}{\sqrt{2}} \\ &+ \frac{\cos^2\theta}{2\sqrt{1-\frac{\sin^2\theta}{2}}} + \frac{\sin^2\theta}{2\sqrt{1-\frac{\cos^2\theta}{2}}} + \frac{k}{k-1}\sin(\frac{3\pi}{4}+\theta) -\frac{k}{k-1}(\frac{\pi}{4} - \theta)\cos(\frac{3\pi}{4} + \theta) \\ &\leq  -(\pi - \arccos{\frac{\cos(\frac{\pi}{4})}{\sqrt{2}}})\frac{\cos(\frac{\pi}{4})}{\sqrt{2}} -(\pi - \arccos{\frac{\sin(\frac{3}{4})}{\sqrt{2}}})\frac{\sin(\frac{3}{4})}{\sqrt{2}} \\ &+ \frac{\cos^2(\frac{3}{4})}{2\sqrt{1-\frac{\sin^2(\frac{\pi}{4})}{2}}} + \frac{\sin^2(\frac{\pi}{4})}{2\sqrt{1-\frac{\cos^2(\frac{3}{4})}{2}}} + 2\sin(\frac{3\pi}{4}+\frac{3}{4}) -2(\frac{\pi}{4} - \frac{3}{4})\cos(\frac{3\pi}{4} + \frac{3}{4}) < 0
	\end{split}
	\end{equation}

	We conclude that $d_1(\theta) \geq 0$ for all $\theta \in [0,\frac{\pi}{4}]$ as desired.

\end{proof}

\medskip \par \noindent {\it Proof of Proposition \ref{tight_bound_prop}. }
First assume that $w_1 \geq -w_2$. Let $\theta$ be the angle between $\mathbf{w}$ and the positive $x$ axis. Then $\cos\theta = \frac{w_1}{\left\|\mathbf{w}\right\|}$ and $\tan\theta = -\frac{w_2}{w_1}$. Therefore we get $$\cos\theta_{\mathbf{w}_l,\mathbf{w}^*_r} = \frac{w_1}{\left\|\mathbf{w}\right\|\sqrt{2}} = \frac{\cos\theta}{\sqrt{2}}$$ and $$\cos\theta_{\mathbf{w}_r,\mathbf{w}^*_l} = \frac{-w_2}{\left\|\mathbf{w}\right\|\sqrt{2}} = \frac{\cos\theta\tan\theta}{\sqrt{2}} = \frac{\sin\theta}{\sqrt{2}}$$
We can rewrite $\ell(\mathbf{w})$ as 
\begin{equation}
\begin{split}
\ell(\mathbf{w}) &= \frac{1}{k^2}\Bigg[ \frac{k^2-3k+2}{2\pi}(\left\|\mathbf{w}\right\| - \left\|\mathbf{w}^*\right\|)^2 + \frac{k}{2}{\left\|\mathbf{w}\right\|}^2+2(k-1)g(\mathbf{w}_r,\mathbf{w}_l) \\ &- \frac{\left\|\mathbf{w}\right\|\left\|\mathbf{w}^*\right\|}{2\pi}\Big(2k\big(\sin(\frac{3\pi}{4} + \theta) + (\frac{\pi}{4} -\theta)\cos(\frac{3\pi}{4} + \theta)\big)\Big) \\ &+  \big(2k-2\big)\big(\sqrt{1-\frac{{\cos\theta}^2}{2}} + (\pi - \arccos{\frac{\cos\theta}{\sqrt{2}}})\frac{\cos\theta}{\sqrt{2}} \big) + \big(2k-2\big)\big(\sqrt{1-\frac{{\sin\theta}^2}{2}} + (\pi - \arccos{\frac{\sin\theta}{\sqrt{2}}})\frac{\sin\theta}{\sqrt{2}} \big)\Big) \\ &+ \frac{k}{2}{\left\|\mathbf{w}^*\right\|}^2 + 2(k-1)g(\mathbf{w}^*_r,\mathbf{w}^*_l) \Bigg]
\end{split}
\end{equation}

Hence by Lemma \ref{al_ar_lem} and Lemma \ref{trigo_lem} we can lower bound $\ell(\mathbf{w})$ as follows 
\begin{equation}
\begin{split}
\ell(\mathbf{w}) &\geq \frac{1}{k^2}\Bigg[\frac{k^2-3k+2}{2\pi}(\left\|\mathbf{w}\right\| - \left\|\mathbf{w}^*\right\|)^2 + \frac{k}{2}{\left\|\mathbf{w}\right\|}^2 + \frac{k-1}{\pi}\big(\frac{\sqrt{3}}{2} - \frac{\pi}{6}\big){\left\|\mathbf{w}\right\|}^2 \\ &- \frac{(k-1)\left\|\mathbf{w}\right\|\left\|\mathbf{w}^*\right\|}{\pi}\big(\sqrt{3} + \frac{2\pi}{3}\big) + \frac{k}{2}{\left\|\mathbf{w}^*\right\|}^2 +  \frac{k-1}{\pi}\big(\frac{\sqrt{3}}{2} - \frac{\pi}{6}\big){\left\|\mathbf{w}^*\right\|}^2\Bigg]
\end{split}
\end{equation}

By setting $\left\|\mathbf{w}\right\| = \alpha\left\|\mathbf{w}^*\right\|$ we get 

\begin{equation}
\begin{split}
\frac{\ell(\mathbf{w})}{{\left\|\mathbf{w}^*\right\|}^2} &\geq  \frac{1}{k^2}\Bigg[\frac{k^2-3k+2}{2\pi}(\alpha-1)^2 + \frac{k}{2}\alpha^2 + \frac{k-1}{\pi}\big(\frac{\sqrt{3}}{2} - \frac{\pi}{6}\big)\alpha^2 \\ &- \frac{(k-1)}{\pi}\big(\sqrt{3} + \frac{2\pi}{3}\big)\alpha + \frac{k}{2} + \frac{k-1}{\pi}\big(\frac{\sqrt{3}}{2} - \frac{\pi}{6}\big)\Bigg]
\end{split}
\end{equation}

Solving for $\alpha$ that minimizes the latter expression we obtain $$\alpha^* = \frac{\frac{k^2-3k+2}{\pi} + \frac{(k-1)}{\pi}\big(\sqrt{3} + \frac{2\pi}{3}\big)}{k + \frac{k^2-3k+2}{\pi} + \frac{2(k-1)}{\pi}(\frac{\sqrt{3}}{2} - \frac{\pi}{6}\big)} = \frac{h(k)}{h(k)+1}  $$

Plugging $\alpha^*$ back to the inequality we get $$\ell(\mathbf{w}) \geq \frac{1}{k^2}\Big(\frac{h(k)+1}{2}(\alpha^*)^2 - h(k)\alpha^* + \frac{h(k)+1}{2}\Big){\left\|\mathbf{w}^*\right\|}^2 = \frac{2h(k)+1}{k^2(2h(k)+2)}{\left\|\mathbf{w}^*\right\|}^2$$
and for $\tilde{\mathbf{w}} = -\alpha^*\mathbf{w}^*$ it holds that $\ell(\tilde{\mathbf{w}}) = \frac{2h(k)+1}{k^2(2h(k)+2)}{\left\|\mathbf{w}^*\right\|}^2$.

Finally, assume $w_1 \leq -w_2$. In this case, let $\theta$ be the angle between $\mathbf{w}$ and the negative $y$ axis. Then $\cos\theta = \frac{-w_2}{\left\|\mathbf{w}\right\|}$ and $\tan\theta = -\frac{w_1}{w_2}$. Therefore $$\cos\theta_{\mathbf{w}_l,\mathbf{w}^*_r} = \frac{w_1}{\left\|\mathbf{w}\right\|\sqrt{2}} = \frac{\cos\theta\tan\theta}{\sqrt{2}} =  \frac{\sin\theta}{\sqrt{2}}$$ and $$\cos\theta_{\mathbf{w}_r,\mathbf{w}^*_l} = \frac{-w_2}{\left\|\mathbf{w}\right\|\sqrt{2}} =  \frac{\cos\theta}{\sqrt{2}}$$

Notice that from now on we get the same analysis as in the case where $w_1 \geq -w_2$, where we switch between expressions with $\mathbf{w}_l,\mathbf{w}^*_r$ and expressions with $\mathbf{w}_r,\mathbf{w}^*_l$. This concludes our proof.
{\hfill $\square$ \bigskip \par}

\section{Experimental Setup for \secref{general_conv_experiments}}
\label{supp:experimental_setup}

In our experiments we estimated the probability of convergence to the global minimum of a randomly initialized gradient descent for many different ground truths $\mathbf{w}^*$ of a convolutional neural network with overlapping filters. For each value of number of hidden neurons, filter size, stride length and ground truth distribution we randomly selected $30$ different ground truths $\mathbf{w}^*$ with respect to the given distribution. We tested with all combinations of values given in Table \ref{table:experiment}.

Furthermore, for each combination of values of number of hidden neurons, filter size and stride length we tested with deterministic ground truths: ground truth with all entries equal to 1, all entries equal to -1 and with entries that form an increasing sequence from -1 to 1, -2 to 0 and 0 to 2 or decreasing sequence from 1 to -1, 0 to -2 and 2 to 0.

For each ground truth, we ran gradient descent 20 times and for each run we recorded whether it reached a point very close to the unique global minimum or it repeatedly (5000 consecutive iterations) incurred very low gradient values and stayed away from the global minimum. We then calculated the empirical probability $\hat{p}=\frac{\text{\#times reached global minimum}}{20}$. To compute the one-sided confidence interval we used the Wilson method (\citep{brown2001interval}) which gives a lower bound 
\begin{equation}
\label{conf_lower_bound}
\frac{\hat{p} + \frac{z_{\alpha}^2}{2n} + z_{\alpha}\sqrt{\frac{\hat{p}(1-\hat{p})}{n}+\frac{z_{\alpha}^2}{4k^2}}}{1+\frac{z_{\alpha}^2}{n}}
\end{equation}
where $z_{\alpha}$ is the $Z$-score with $\alpha=0.05$ and in our experiments $n=20$. Note that we initialized gradient descent inside a large hypercube such that outside the hypercube the gradient does not vanish (this can be easily proved after writing out the gradient for each setting).

For all ground truths we got $\hat{p} \geq 0.15$, i.e., for each ground truth we reached the global minimum at least $3$ times. Hence the confidence interval lower bound \eqref{conf_lower_bound} is greater than $\frac{1}{17}$ in all settings. This suggests that with a few dozen repeated runs of a randomly initialized gradient descent, with high probability it will converge to the global minimum.
\begin{table}[t]
	\caption{Parameters values for experiments in \secref{general_conv_experiments}}
	\begin{center}
		
		\begin{tabular}{ |c| c| c| }
			
			\hline
			Number of hidden neurons & 50,100  \\ 
			\hline
			Filter size & 2,8,16  \\ 
			\hline
			stride length & 1,$\min\{\frac{f}{4},1\},\min\{\frac{f}{2},1\}$ where $f$ is the filter size  \\
			&(For instance, for $f=16$ we used strides 1,4,8 \\
			& and for $f=2$ we used stride 1) \\
			\hline
			Ground truth distribution & The entries of the ground truth are i.i.d. \\
			& uniform random variables over the interval $[a,b]$ \\
			& where $(a,b)\in\{(-1,1),(-2,0),(0,2)\}$ \\
			\hline
		\end{tabular}
		\label{table:experiment}
	\end{center}
\end{table}

\section{Uniqueness of Global Minimum in the Population Risk}
\label{supp:min_prop_section}
Without loss of generality we assume that the filter is of size $2$ and the stride is $1$. The proof of the general case follows the same lines. Assume that $\ell(\mathbf{w}) = 0$ and denote $\mathbf{w}=(w_1,w_2)$, $\mathbf{w}^*=(w^*_1,w^*_2)$. Recall that $\ell(\ww) = \expect{\cG}{(f(\xx;W) - f(\xx;W^*))^2}$ where $f(\xx;W) = {1\over k} \sum_{i} \relu{\ww_i\cdot\xx }$ and for all $1 \leq i \leq k$ $\ww_i = (\mathbf{0}_{i-1},\ww,\mathbf{0}_{d-i-1})$. By equating $\ell(\mathbf{w})$ to $0$ we get that  $(f(\xx;W) - f(\xx;W^*))^2=0$ almost surely. Since $(f(\xx;W) - f(\xx;W^*))^2$ is a continuous function it follows that $f(\xx;W) - f(\xx;W^*) = 0$ for all $\xx$. In particular this is true for $\xx_1=(x,0,0,...,0)$, $x\in \mathbb{R}$. Thus $\relu{xw_1} = \relu{xw^*_1}$ for all $x \in \mathbb{R}$ which implies that $w_1 = w^*_1$. The equality holds also for $\xx_2=(0,x,0,...,0)$, $x\in \mathbb{R}$ which implies that $\relu{xw_2} + \relu{xw_1} = \relu{xw^*_2} + \relu{xw^*_1}$ for all $x \in \mathbb{R}$. By the previous result, we get $\relu{xw_2} = \relu{xw^*_2}$ for all $x \in \mathbb{R}$ and thus $w_2 = w^*_2$. We proved that $\mathbf{w}=\mathbf{w}^*$ and therefore $\mathbf{w}^*$ is the unique global minimum.

%\begin{thebibliography}{9}
	
%	\bibitem{Nesterov}
%	Y.Nesterov. \textit{Introductory lectures on convex optimization}, pages 22-–29, 2004. 
	
%\end{thebibliography}
\end{document}

%% file: macros.tex
\usepackage{boxedminipage}

\renewcommand{\xi}{{\xx}^{(m)}}

\newcommand{\needcite}[1]{}

\newcommand{\be}{\begin{equation}}
\newcommand{\ee}{\end{equation}}
\newcommand{\benn}{\begin{equation*}}
\newcommand{\eenn}{\end{equation*}}
\newcommand{\bea}{\begin{eqnarray*}}
\newcommand{\eea}{\end{eqnarray*}}
\newcommand{\bean}{\begin{eqnarray}}
\newcommand{\eean}{\end{eqnarray}}

\newcommand{\ww}{\boldsymbol{w}} 
\newcommand{\xx}{\boldsymbol{x}}

\newcommand{\vv}{\boldsymbol{v}} 
\newcommand{\dd}{\boldsymbol{d}} 
\newcommand{\eee}{\boldsymbol{e}} 
\newcommand{\uu}{\boldsymbol{u}} 
\newcommand{\xpart}[1]{\xx[#1]}
\newcommand{\cD}{{\cal D}}
\newcommand{\cG}{{\cal G}}
\newcommand{\ourarch}{\em No-Overlap Networks}
\newcommand{\ourarchsngl}{\em No-Overlap Network}
\newcommand{\relu}[1]{\sigma\left(#1\right)}
\renewcommand{\aa}{\boldsymbol{a}}

\newcommand{\ignore}[1]{}

\newcommand{\polyring}[1]{\reals\left[x_1,\ldots,x_n\right]}

\newcommand{\zero}{\boldsymbol{0}}

{
\begin{center}
\begin{boxedminipage}{0.8\linewidth}
\begin{center}
\textbf{\texttt{#1}}
\end{center}
\rm
\begin{tabbing}
....\=...\=...\=...\=...\=  \+ \kill
} %
{\end{tabbing} 
\end{boxedminipage} \end{center} %\vspace{0.3cm}
}

\renewcommand{\eqref}[1]{Eq.~\ref{#1}}

\newcommand{\figref}[1]{Figure \ref{#1}}
\newcommand{\secref}[1]{Section \ref{#1}}

\newcommand{\expect}[2]{\mathbb{E}_{#1}\left[{#2}\right]}
\newcommand{\reals}{\mathbb{R}}

\newtheorem{theorem}{Theorem}

\newtheorem{definition}[theorem]{Definition}